\documentclass{article}

    \PassOptionsToPackage{numbers, compress}{natbib}

    \usepackage[final]{neurips_2024}

\usepackage[utf8]{inputenc} 
\usepackage[T1]{fontenc}    
\usepackage{hyperref}       
\usepackage{url}            
\usepackage{booktabs}       
\usepackage{amsfonts}       
\usepackage{nicefrac}       
\usepackage{microtype}      
\usepackage{xcolor}         
\usepackage{graphicx}
 \graphicspath{{figures/}} 
\usepackage{float}
\newfloat{figtab}{htb}{fgtb}
\makeatletter
  \newcommand\figcaption{\def\@captype{figure}\caption}
  \newcommand\tabcaption{\def\@captype{table}\caption}

\usepackage{amsmath}
\usepackage{amssymb}
\usepackage{amsthm}
\usepackage{subfigure}
\usepackage{wrapfig}        
\usepackage{algorithm}
\usepackage{algorithmic}
\usepackage{xcolor}
\usepackage{CJKutf8}

\usepackage{multirow}
\usepackage{bbm}
\usepackage[inline]{enumitem}
\usepackage[capitalize,noabbrev]{cleveref}

\newtheorem{proposition}{Proposition}
\newtheorem{definition}{Definition}

\hypersetup{
	colorlinks=true
}

\newcommand{\argmax}{\mathop{\mathrm{argmax}}}
\newcommand{\argmin}{\mathop{\mathrm{argmin}}}

\newcommand{\E}{\mathbb E}

\newcommand{\Acal}{\mathcal A}
\newcommand{\Dcal}{\mathcal D}
\newcommand{\Mcal}{\mathcal M}
\newcommand{\Ncal}{\mathcal N}
\newcommand{\Scal}{\mathcal S}

\usepackage{pifont}
\usepackage{xspace}
\def\algo{SCAS\xspace}

\title{Offline Reinforcement Learning with OOD State Correction and OOD Action Suppression}

\author{
  Yixiu Mao$^{1}$, Qi Wang$^{1}$, Chen Chen$^{1}$, Yun Qu$^{1}$, Xiangyang Ji$^{1}$\\
  $^{1}$Department of Automation, Tsinghua University \\
  \texttt{myx21@mails.tsinghua.edu.cn,} \texttt{xyji@tsinghua.edu.cn}
}

\begin{document}

\maketitle

\begin{abstract}
In offline reinforcement learning (RL), addressing the out-of-distribution (OOD) action issue has been a focus, but we argue that there exists an OOD state issue that also impairs performance yet has been underexplored. Such an issue describes the scenario when the agent encounters states out of the offline dataset during the test phase, leading to uncontrolled behavior and performance degradation. To this end, we propose SCAS, a simple yet effective approach that unifies OOD state correction and OOD action suppression in offline RL. Technically, SCAS achieves value-aware OOD state correction, capable of correcting the agent from OOD states to high-value in-distribution states. Theoretical and empirical results show that SCAS also exhibits the effect of suppressing OOD actions. On standard offline RL benchmarks, SCAS achieves excellent performance without additional hyperparameter tuning. Moreover, benefiting from its OOD state correction feature, SCAS demonstrates enhanced robustness against environmental perturbations.
\end{abstract}

\section{Introduction}

Deep reinforcement learning~(RL) shows promise in solving sequential decision-making problems, gaining increasing interest for real-world applications~\citep{mnih2015human,silver2017mastering,vinyals2019grandmaster,schrittwieser2020mastering,degrave2022magnetic}. However, deploying RL algorithms in extensive scenarios poses persistent challenges, such as risk-sensitive exploration~\cite{garcia2015comprehensive} and time-consuming episode collection~\cite{kober2013reinforcement}. Recent advances view offline RL as a hopeful solution to these challenges~\citep{levine2020offline}. Offline RL aims to learn a policy from a fixed dataset without further interactions~\citep{lange2012batch}. It can tap into existing large-scale datasets for safe and efficient learning~\citep{johnson2016mimic,maddern20171,qu2024hokoff}.

In offline RL research, a well-known concern is the out-of-distribution (OOD) action issue: the evaluation of OOD actions causes extrapolation error~\cite{fujimoto2019off}, which can be exacerbated by bootstrapping and result in severe value overestimation~\cite{levine2020offline}.
To address this issue, a large body of work has emerged to directly or indirectly \textit{suppress OOD actions} during training, employing various techniques such as policy constraint~\citep{fujimoto2019off,kumar2019stabilizing,fujimoto2021minimalist}, value penalization~\citep{kumar2020conservative,an2021uncertainty,cheng2022adversarially}, and in-sample learning~\citep{kostrikov2022offline,garg2023extreme,xu2023offline}.

\begin{figure}
    \centering
    \subfigure[CQL]{
    \label{fig:tsne_4fig_cql}
    \includegraphics[width=0.228\textwidth]{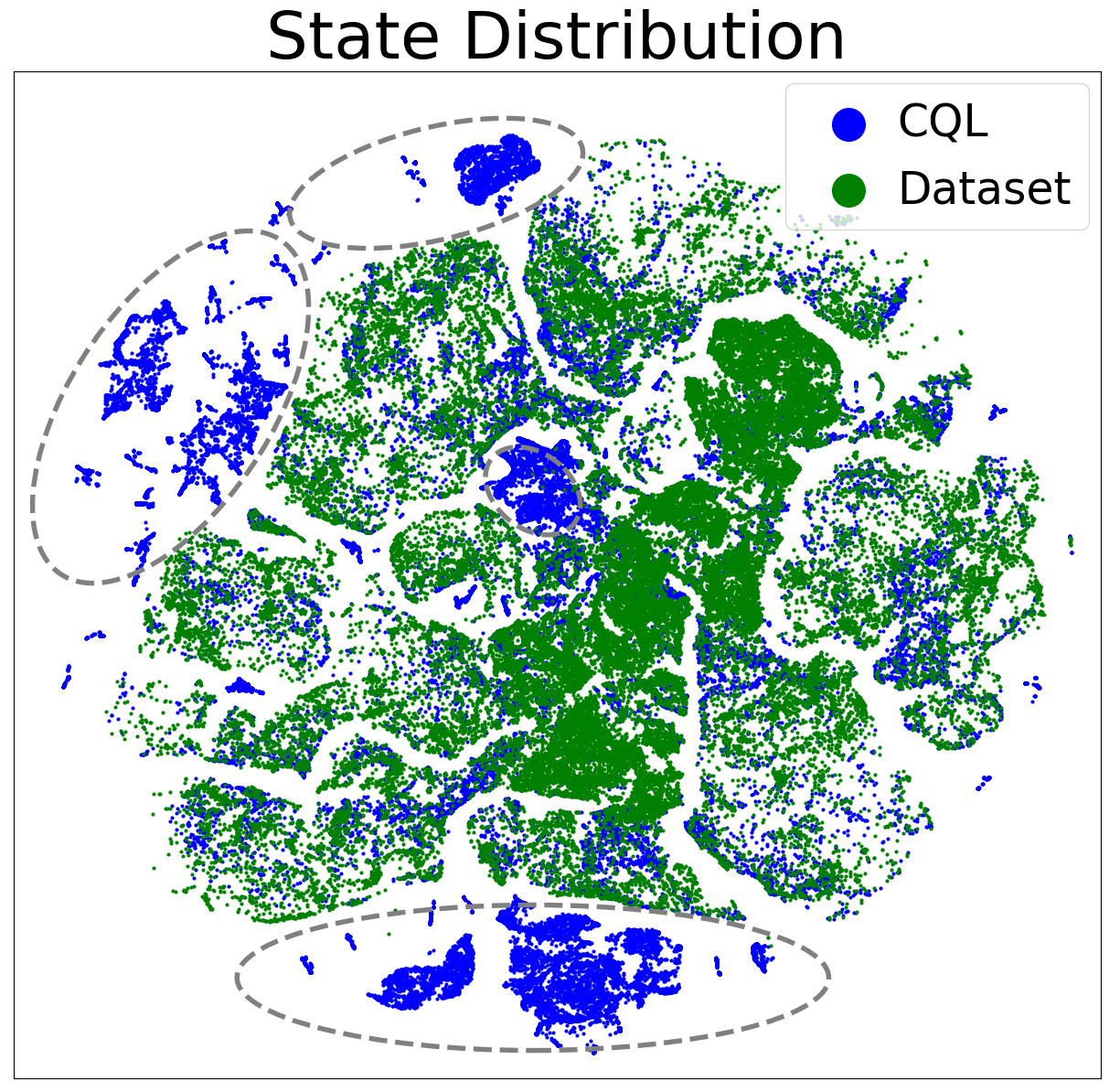}
    }\hspace{-2mm}
    \subfigure[TD3BC]{
    \label{fig:tsne_4fig_td3bc}
    \includegraphics[width=0.228\textwidth]{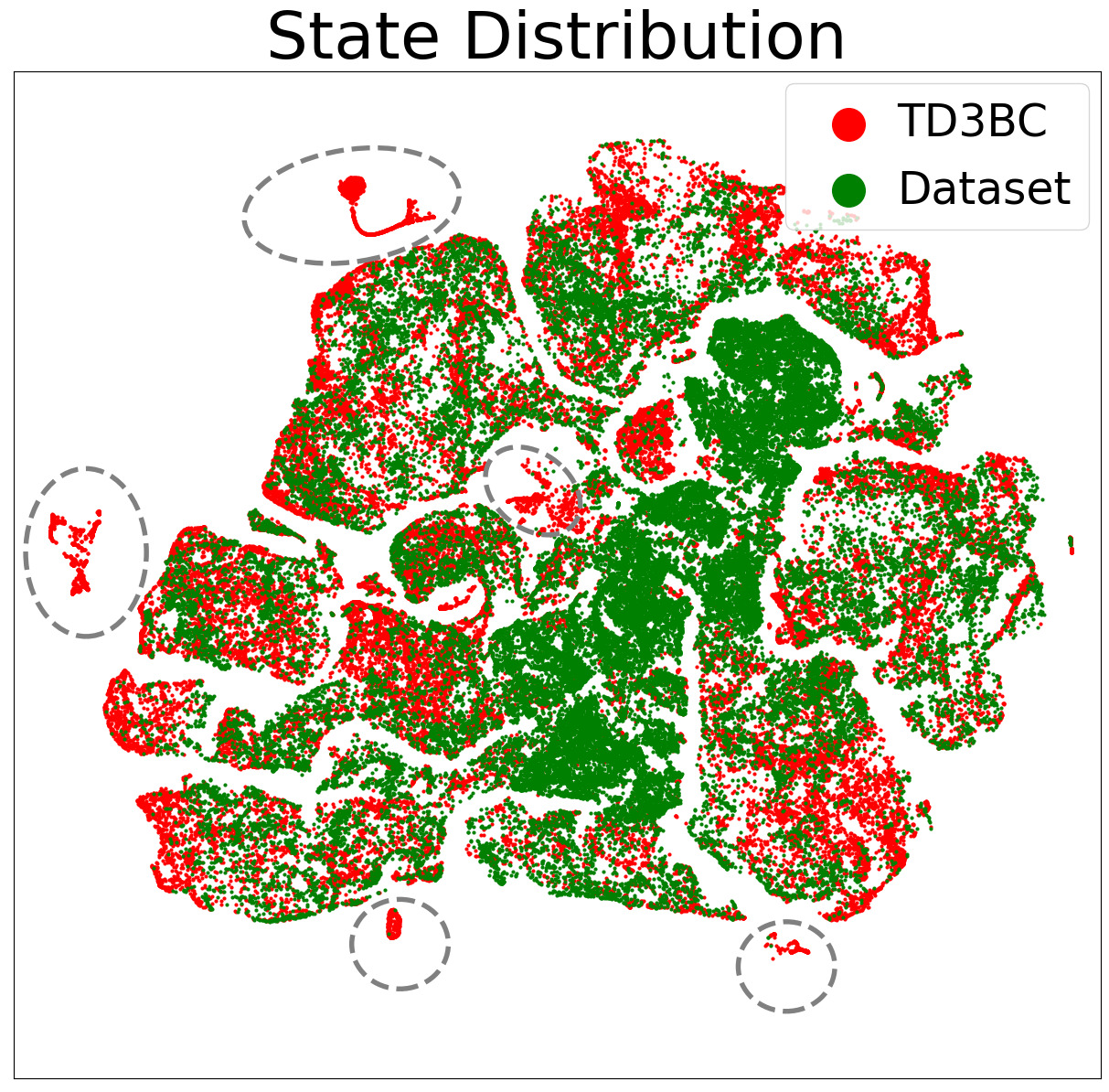}
    }\hspace{-2mm}
    \subfigure[SCAS~(Ours)]{
    \label{fig:tsne_4fig_scas}
    \includegraphics[width=0.228\textwidth]{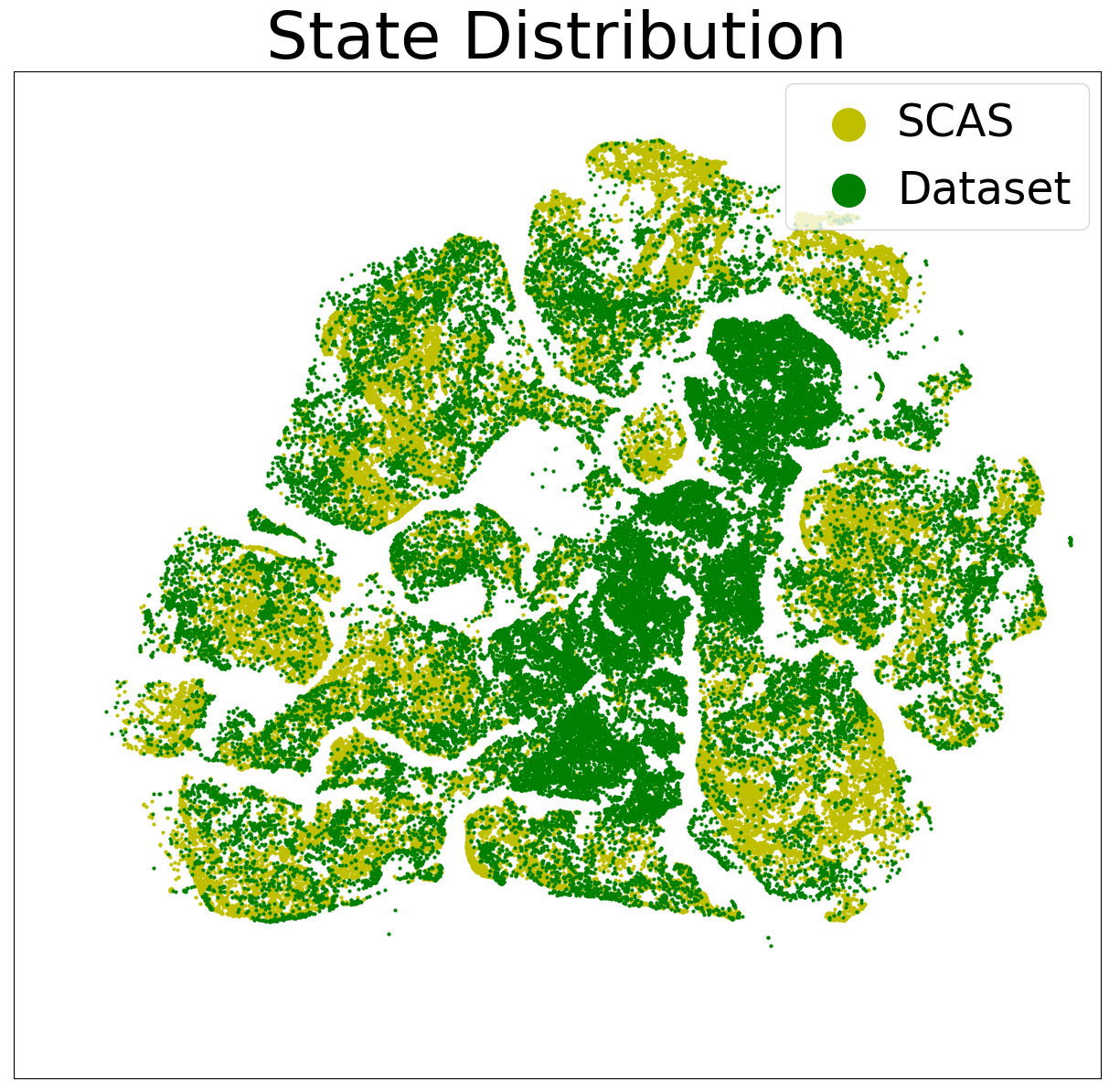}
    }\hspace{-2mm}
    \subfigure[Optimal Value]{
    \label{fig:tsne_4fig_value}
    \includegraphics[width=0.269\textwidth]{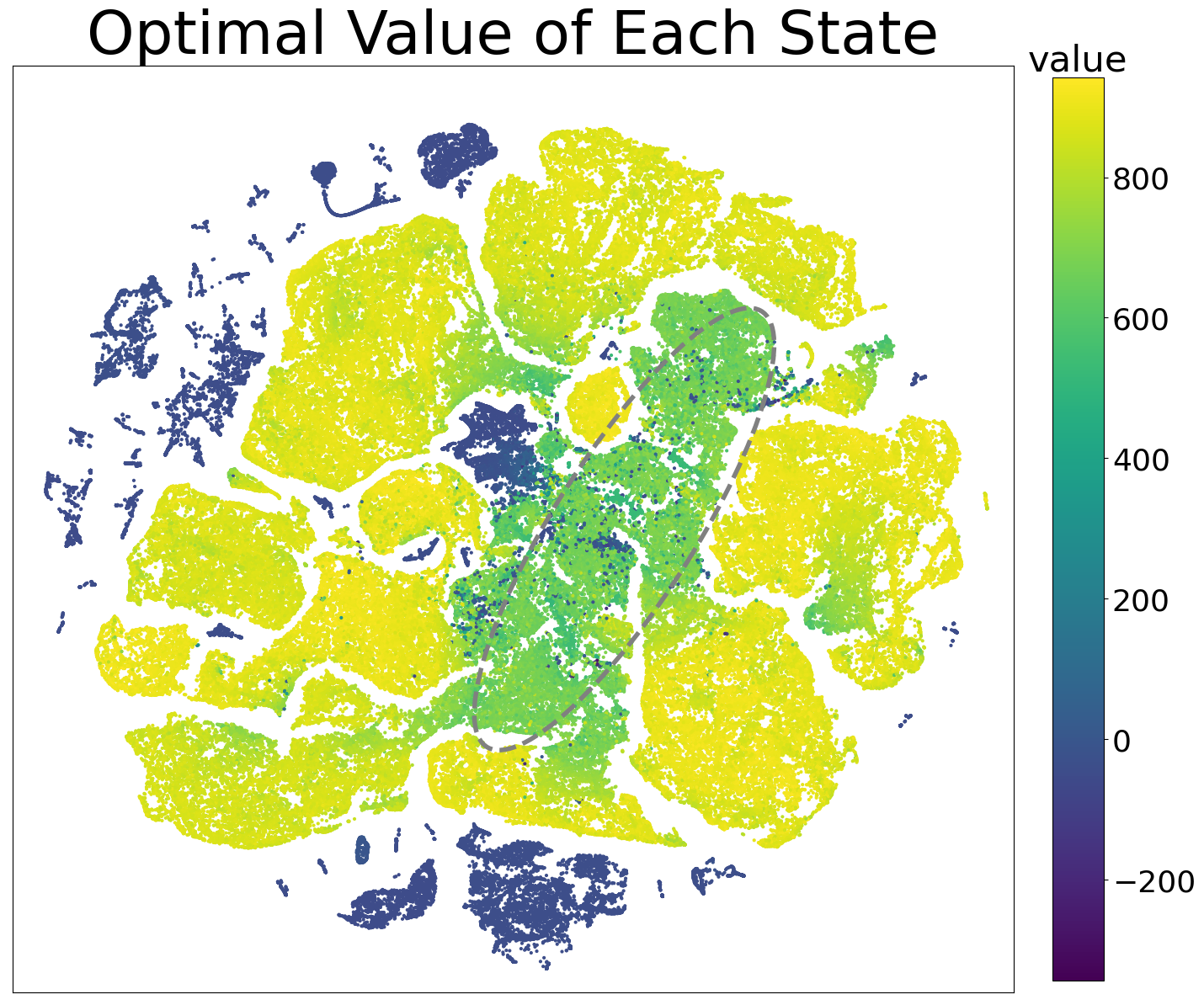}
    }\hspace{-2mm}
    \caption{\textbf{The resulting state distributions of offline RL algorithms and optimal values of states.}
(a,b,c)~The state distributions generated by the learned policies of various algorithms compared with that of the offline dataset on halfcheetah-medium-expert.
(d)~The corresponding optimal value of each state, which is obtained by running TD3 online to convergence.
\textit{\algo-induced state distribution is almost entirely within the support of the offline distribution and avoids the low-value areas}, while CQL and TD3BC tend to produce OOD states with extremely low values.}
    \label{fig:tsne_4fig}
    \vspace{-2mm}
\end{figure}

Distinguished from most previous works, this paper argues that, apart from the OOD action issue, there exists an \textit{OOD state issue} that also impairs performance yet has received limited attention in the field.
Such an issue refers to the agent encountering states out of the offline dataset during the policy deployment phase~(i.e., test phase).
The occurrence of OOD states can be attributed to OOD actions, stochastic environments, and real-world perturbations.
Since typical offline RL algorithms do not involve policy training in OOD states, the agent tends to behave in an uncontrolled manner once entering OOD states in the test phase. This can further exacerbate the state deviation from the offline dataset and lead to severe degradation in performance~\citep{levine2020offline,zhang2022state}.

In mitigating this OOD state issue, existing limited work attempts to train the policy to correct the agent from OOD states to in-distribution~(ID) states~\citep{zhang2022state,jiang2023recovering}.
Technically, \citet{zhang2022state} construct a dynamics model and a state transition model and align them to guide the agent to ID regions, while \citet{jiang2023recovering} resort to an inverse dynamics model for policy constraint.
However, they deal with the OOD state and OOD action issues separately, requiring extra OOD action suppression components and complex distribution modeling, which sacrifices computational efficiency and algorithmic simplicity.
Moreover, correcting the agent to all ID states impartially could be problematic, especially when the dataset contains substantial suboptimal states. 
As a result, the performance of prior methods also leaves considerable room for improvement.

In this paper, we aim to address these two fundamental OOD issues simultaneously by proposing a simple yet effective approach for offline RL.
We term our method SCAS due to its integration of OOD \underline{\textbf{S}}tate \underline{\textbf{C}}orrection and OOD \underline{\textbf{A}}ction \underline{\textbf{S}}uppression.
We start with solving an analytical form of a value-aware state transition distribution, which is within the dataset support but skewed toward high-value states. Then, we align it with the dynamics induced by the trained policy on perturbed states via KL divergence. 
This operation intends to correct the agent from OOD states to high-value ID states, a concept we refer to as \textit{value-aware} OOD state correction.
Through some derivations, it also eliminates the necessity of training a multi-modal state transition model.
Furthermore, we show theoretically and empirically that, while designed for OOD state correction,
\algo regularization also exhibits the effect of OOD action suppression.
We evaluate \algo on the offline RL benchmarks including D4RL~\citep{fu2020d4rl} and NeoRL~\citep{qin2022neorl}. \algo achieves excellent performance with consistent hyperparameters without additional tuning. Moreover, benefiting from its OOD state correction ability, \algo demonstrates improved robustness against environmental perturbations.

To summarize, the main contributions of this work are:
\begin{itemize}[leftmargin=*]
\item We systematically analyze the underexplored OOD state issue in offline RL and propose a simple yet effective approach \algo \textit{unifying OOD state correction and OOD action suppression}.
\item Our approach achieves \textit{value-aware} OOD state correction, which circumvents modeling complex distributions and significantly improves performance over vanilla OOD state correction methods.
\item Empirically\footnote{Our code is available at \href{https://github.com/maoyixiu/SCAS}{https://github.com/maoyixiu/SCAS}.}, our approach demonstrates superior performance on standard offline RL benchmarks and enhanced robustness in perturbed environments \textit{without additional hyperparameter tuning}.
\end{itemize}

\section{Preliminaries}
\label{preliminaries}
In reinforcement learning, we generally characterize the environment as a Markov Decision Process~(MDP) $\Mcal=(\mathcal{S}, \mathcal{A}, P, R, \gamma, d_0)$, with state space $\mathcal{S}$, action space $\mathcal{A}$, transition dynamics $P: \Scal \times \Acal \to \Delta(\Scal)$, reward function $R: \Scal \times \Acal \to \mathbb{R}$, discount factor $\gamma \in [0,1)$, and initial state distribution $d_0$~\citep{sutton2018reinforcement}. 
The agent interacts with the environment and seeks a policy $\pi: \Scal \to \Delta(\Acal)$ to maximize the expected discounted return $\eta(\pi)$:
\begin{equation}
\eta(\pi) = \mathbb{E}_{s_0\sim d_0, a_t\sim\pi(\cdot|s_t), s_{t+1}\sim P(\cdot|s_t,a_t)}\left[\sum_{t=0}^\infty\gamma^t R(s_t,a_t)\right].
\end{equation}
For any policy $\pi$, we define the value function as $V^\pi(s)=\mathbb{E}_\pi\left[\sum_{t=0}^{\infty}\gamma^t R(s_t,a_t) | s_0=s\right]$ and the state-action value function~($Q$-value function) as $Q^\pi(s,a)=\mathbb{E}_\pi\left[\sum_{t=0}^{\infty}\gamma^t R(s_t,a_t) | s_0=s,a_0=a\right]$.

\paragraph{Offline RL.} 
In offline RL, the agent 
can only access a static dataset $\Dcal=\{(s_{t}^{i},a_{t}^{i},s_{t+1}^{i},r_{t}^{i})\}$.
We denote the empirical behavior policy of $\Dcal$ by $\beta(a|s)$ and the empirical dynamics model by $M  (s'|s,a)$, both of which depict the conditional distributions observed in the dataset~\citep{fujimoto2019off}.
Typical actor-critic algorithms~\citep{silver2014deterministic,haarnoja2018soft} evaluate policy $\pi$ by minimizing Bellman loss:
\begin{equation}
\begin{aligned}
\label{eqn:Q loss}
L_{Q}(\theta) = \mathbb{E}_{(s, a, s')\sim \mathcal{D}}[(Q_\theta(s,a) - R(s,a) - \gamma \mathbb{E}_{a'\sim \pi_\phi(\cdot|s')}Q_{\theta'}(s',a'))^2],
\end{aligned}
\end{equation}
where $\pi_\phi$ and $Q_\theta$ are the parameterized policy and $Q$ function, and $Q_{\theta'}$ is a target network whose parameters are updated via Polyak averaging~\cite{mnih2015human}.

Simultaneously, policy improvement in policy iteration is achieved via maximizing the Q-value:
\begin{equation}
L_{\pi}(\phi) = -\mathbb{E}_{s \sim \mathcal{D}, a \sim \pi_\phi}\left[Q_{\theta}\left(s, a\right)\right].
\end{equation}

\paragraph{OOD action issue.}
In offline RL, \textit{OOD actions} refer to actions outside the support of the behavior policy $\beta(\cdot|s)$ at a specific state $s \in \Dcal$. Since the Q-values of OOD actions can be poorly estimated and the policy improvement is towards maximizing the estimated $Q_\theta$, the resulting policy tends to prioritize the OOD actions with overestimated values, leading to poor performance~\citep{fujimoto2019off}.

\section{OOD State Correction}
\label{sec:OOD State Correction}
The following focuses on the OOD state issue and OOD state correction in offline RL. In \cref{OOD State Issue in Offline RL}, we systematically analyze the OOD state issue, introduce the concept of OOD state correction, and point out limitations of prior methods. Then we present the proposed approach \algo in \cref{sec:Simple and Value-aware OOD State Correction}.
\subsection{OOD State Issue in Offline RL}
\label{OOD State Issue in Offline RL}

In offline RL, \textit{OOD states} refer to states not in the offline dataset.
The OOD state issue~(\cref{def:OOD state issue}) pertains to scenarios where the agent enters OOD states during the test phase, potentially resulting in catastrophic failure~\citep{levine2020offline}.
However, such a topic is rarely investigated in the literature, and existing studies lack deep insights.
We mathematically formulate the OOD state issue as follows.
\begin{definition}[OOD state issue]
\label{def:OOD state issue}
There exists $s \in \mathcal{S}$, such that $d_\mathcal {M_T}^\pi(s)>0$ and $d_\mathcal D(s)=0$, where $\mathcal {M_T}$ is the MDP of the test environment, $\pi$ is any learned policy, $d_\mathcal {M_T}^\pi$ is the state probability density induced by $\pi$ in $\mathcal {M_T}$, and $d_\mathcal D$ is the state probability density in the offline dataset.
\end{definition}

\paragraph{Origins and consequence of OOD states.}
During the test phase, the OOD states occur primarily in three scenarios:
(i) OOD actions: the learned policy, not perfectly constrained within the support of the behavior policy, executes unreliable OOD actions, leading to  OOD states. 
(ii) Stochastic environment: the initial state of the actual environment may fall outside the offline dataset.
In addition, stochastic dynamics can also lead to states outside the dataset, even when taking ID actions in ID states.
(iii) Perturbations: commonly seen in real-world robot applications, some unexpected perturbations can propel the agent into OOD states~(e.g., wind, human interference).

During offline training, the typical Bellman updates involve only ID states, and the policies in OOD states are not trained.
As a result, when encountering OOD states in the test phase, the agent would exhibit uncontrolled behavior, and the state deviation from the offline dataset can be further exacerbated over time steps, severely degrading performance~\citep{levine2020offline}.

\paragraph{OOD state correction.}
To mitigate this OOD state issue, an intuitive solution is to train a policy capable of correcting the agent from OOD states to ID states, a concept known as \textit{OOD state correction}~\citep{zhang2022state}. Specifically, during offline training, we can perturb the original state $s$ in the dataset into $\hat s$ to generate substantial OOD states. Then consider the scenario where the agent starts from $\hat s$, follows the trained policy $\pi$, and transitions to the next state $\hat s'$. To reduce state deviation, $\hat s'$ is expected to be close to the offline dataset. Thus we can align the distribution of $\hat s'$ with an ID state distribution to regularize the policy and achieve OOD state correction.

Continuing the above train of thought, 
SDC~\citep{zhang2022state} generates the ID state distribution by feeding the original state $s$ into a trained state transition model $N  (s'|s)$ of the dataset.
This model characterizes the conditional state transition distribution in the dataset and is implemented by a conditional variational auto-encoder~(CVAE)~\citep{sohn2015learning}.
After pretraining a dynamics model $M  (s'|s,a)$ and the state transition model $N  (s'|s)$, SDC introduces the following policy regularizer for OOD state correction:
\begin{equation}
\min_\pi \underset{s \sim \Dcal}{\E}~ \underset{\hat s \sim \Ncal_\sigma(s)}{\E} \left[ \operatorname{MMD}(M  (\cdot|\hat s, \pi(\cdot|\hat s)), N  (\cdot|s) )\right],
\end{equation}
where $\hat s$ is a Gaussian noise perturbed version of the original state $s$, $\sigma$ is the standard deviation of the Gaussian, $M  (\cdot|\hat s, \pi(\cdot|\hat s))$ is shorthand for $\E_{\hat a \sim \pi(\cdot|\hat s)} M  (\cdot|\hat s, \hat a)$, and MMD is the maximum mean discrepancy measure.
More recently, OSR~\citep{jiang2023recovering} directly aligns the trained policy distribution at the perturbed state with a CVAE inverse dynamics model to constrain the policy in OOD states.

\paragraph{Limitations.}
However, the regularizers of prior methods are only designed to deal with this OOD state issue. To mitigate OOD actions, they require an additional conservative Q learning~(CQL) term~\citep{kumar2020conservative} in value estimation to penalize Q-values of OOD actions.
In addition, the state transition distribution and the inverse dynamics distribution are multi-modal in many scenarios~\citep{moerland2023model}.
The necessity of extra OOD action suppression components and complex distribution modeling compromises their computational efficiency and algorithmic simplicity.
Moreover, correcting the agent to all ID states impartially could be problematic, particularly when the offline dataset contains a large portion of suboptimal states.
In such cases, vanilla OOD state correction can lead to suboptimal behaviors.
Consequently, there is also significant potential for improvement in the performance of prior methods.

For a more comprehensive discussion of related work, please refer to \cref{related works}.

\subsection{Value-aware OOD State Correction}
\label{sec:Simple and Value-aware OOD State Correction}

The objective of this work is to formulate a simple yet effective policy regularizer for offline RL that unifies OOD state correction and OOD action suppression. Moreover, we aim to achieve \textit{value-aware} OOD state correction,
involving the correction of the agent from OOD states to high-value ID states.

\paragraph{Value-aware state transition.}
For the ID state distribution to which the agent is corrected,
we expect a value-aware state transition distribution $N^*(\cdot|s)$ that lies within the support of the dataset state transition distribution $N  (\cdot|s)$ but is skewed toward high-value states $s'$.
To ensure stability and, more importantly, to enable our subsequently designed algorithm to circumvent modeling complex distributions, we seek a soft optimal version of it.
To this end, we consider the following problem\footnote{
Note that the regularizer $\mathrm{D_{KL}}(N^*(\cdot|s) \| N  (\cdot|s))$ can constrain the support of $N^*(\cdot|s)$ within that of $N  (\cdot|s)$,
because if $\operatorname{supp}(N^*(\cdot|s)) \nsubseteq \operatorname{supp}(N  (\cdot|s))$ at some state $s$, then $\mathrm{D_{KL}}(N^*(\cdot|s) \| N  (\cdot|s)) = \infty$.}:
\begin{equation}
\label{eq:optimize N}
\max_{N^*} ~ \underset{s \sim \Dcal}{\E} \left[\alpha   \underset{s'\sim N^*(\cdot|s)}{\E} V(s') -  \mathrm{D_{KL}}(N^*(\cdot|s) \| N  (\cdot|s))\right],
\end{equation}
where $\alpha  $ is a hyperparameter to balance the two terms.

The optimization problem above has a closed-form solution:
\begin{equation}
\label{eq:N soltion}
N^*(s'|s) = \frac{1}{Z(s)}\exp\left(\alpha   V\left(s'\right)\right) N  (s'|s),
\end{equation}
where $Z(s)=\sum_{s'} \exp\left(\alpha   V\left(s'\right)\right) N  (s'|s)$ is a normalization factor. It can be seen from Eq.~\eqref{eq:N soltion} that $\operatorname{supp}(N^*(\cdot|s)) \subseteq \operatorname{supp}(N  (\cdot|s))$. 
Note that $\alpha  $ is a key hyperparameter that controls the significance of the values of next states in \algo's OOD state correction. As $\alpha  $ increases, $N^*(\cdot|s)$ becomes more skewed toward the optimal $s'$ in the support of $N(\cdot|s)$.

\paragraph{OOD state correction.}
In order to produce substantial OOD states, we perturb each state $s \in \Dcal$ with Gaussian noise $\Ncal(0,\sigma^2)$, resulting in perturbed state $\hat s$.
It is worth noting that the dataset used for RL training remains unchanged. We perturb the states solely to formulate the regularizer.

We anticipate the following value-aware OOD state correction scenario, where the agent starts from OOD state $\hat s$, follows the trained policy $\pi$, and transitions to the high-value ID state $s'$ in the distribution of $N^*(\cdot|s)$.
To this end, we train the policy $\pi$ to align the dynamics induced by $\pi$ on the perturbed state $\hat s$ with the value-aware state transition distribution at the original state $s$ via KL divergence. 
That is, we regularize $\pi$ by minimizing:
\begin{equation}
\label{eq:reg KL}
\min_\pi \underset{s \sim \Dcal}{\E}~ \underset{\hat s \sim \Ncal_\sigma(s)}{\E} \mathrm{D_{KL}}(N^*(\cdot|s) \| M  (\cdot|\hat s, \pi(\cdot|\hat s))).
\end{equation}

By substituting the analytical solution of $N^*$ from Eq.~\eqref{eq:N soltion} into the KL divergence, we have
\begin{equation*}
\argmin_\pi \mathrm{D_{KL}}(N^*(\cdot|s) \| M  (\cdot|\hat s, \pi(\cdot|\hat s))) = \argmax_\pi \underset{s' \sim N(\cdot|s)}{\E}~ \left[ \frac{\exp\left(\alpha   V\left(s'\right)\right)}{Z(s)} \log M(s'|\hat s, \pi(\cdot|\hat s)) \right].
\end{equation*}

Note that $N  $ is the state transition distribution in the dataset, and $s\sim\Dcal,s'\sim N  (\cdot|s)$ is equivalent to $(s,s')\sim \Dcal$. 
Thus minimizing Eq.~\eqref{eq:reg KL} is equivalent to maximizing following regularizer:
\begin{equation}
\label{eq:reg IS}
R(\pi) = \underset{(s,s') \sim \Dcal}{\E}~ \underset{\hat s \sim \Ncal_\sigma(s)}{\E} \left[\frac{\exp\left(\alpha   V\left(s'\right)\right)}{Z(s)} \log M  (s'|\hat s, \pi(\cdot|\hat s))\right].
\end{equation}

As a result, $R(\pi)$ effectively eliminates the need for a pre-trained multi-modal state transition model~($N$ or $N^*$) and enables direct sampling from the dataset for optimization.

However, the normalization factor $Z(s)$ in $R(\pi)$ can be challenging to compute. We note that the regularizer $R(\pi)$ is derived from the minimization of the KL divergence in Eq.~\eqref{eq:reg KL}. 
Since we aim to minimize this KL at every state $s$ in $\Dcal$
and $Z(s)$ only affects the relative weights at different $s$,
it matters less to precisely restore the correct state weights in $\Dcal$ by computing $Z(s)$, which is empirically hard to estimate and may bring more instability. 
Thus, we replace $Z(s)$ in $R(\pi)$ with an empirical normalizer $\exp(\alpha   V(s))$ for computational stability:
\begin{equation}
\label{eq:reg mle}
 R_1(\pi) = \underset{(s,s') \sim \Dcal}{\E}~ \underset{\hat s \sim \Ncal_\sigma(s)}{\E} \left[\frac{\exp\left(\alpha   V\left(s'\right)\right)}{\exp\left(\alpha   V\left(s\right)\right)} \log M  (s'|\hat s, \pi(\cdot|\hat s))\right].
\end{equation}

We provide further rationale behind this choice of the empirical normalizer in \cref{appsec:normalizer}.

\paragraph{Tractable optimization.} 
Now we shift focus to the optimization of $ R_1(\pi)$. The expectation with respect to $\pi$ can be moved outside the logarithm by Jensen's inequality:
\begin{equation}
\begin{aligned}
\label{eq:reg Jensen}
R_1(\pi)  \geq \underset{(s,s') \sim \Dcal}{\E}~ \underset{\hat s \sim \Ncal_\sigma(s)}{\E} ~\underset{a \sim \pi(\cdot|\hat s)}{\E} \left[\frac{\exp\left(\alpha   V\left(s'\right)\right)}{\exp\left(\alpha   V\left(s\right)\right)} \log M  (s'|\hat s, a)\right],
\end{aligned}
\end{equation}

where the equality holds when $\pi$ is deterministic. In general, it is convenient to maximize the lower bound in Eq.~\eqref{eq:reg Jensen} using the reparameterization trick.
However, to ensure the equality case in Eq.~\eqref{eq:reg Jensen}, we opt to train a deterministic policy $\pi$. In this case, we can directly maximize $R_1(\pi)$ by computing the gradient of $\pi$ using automatic differentiation~\citep{paszke2017automatic}.

In contrast to model-based RL methods that typically use the learned dynamics model to roll out multi-step trajectories for policy training~\citep{janner2019trust,yu2020mopo}, our algorithm utilizes the dynamics model to propagate the gradient of policy and regularize policy training, resulting in significantly enhanced computational efficiency. Moreover, the nature of one-step dynamics prediction in our method is advantageous for maintaining relatively high prediction accuracy.

\section{Analysis of OOD Action Suppression}
\label{sec:OOD Action Suppression}

This section focuses on the OOD action issue and shows that the proposed regularizer also exhibits the effect of \textit{OOD action suppression}. In other words, it can also prevent the policy from taking OOD actions, thereby simultaneously addressing the fundamental OOD action issue in offline RL.
In offline RL, OOD actions are exclusively defined on ID states. This is because actor-critic training is limited to ID states, and any actions on OOD states would not affect training and cause the OOD action issue mentioned in \cref{preliminaries}.
Consequently, for the analysis of OOD actions, it is essential to consider ID states.
We define $\bar R,\bar R_1$ as the ID state version of $R, R_1$, where $\hat s = s$.
$\bar R$ and $\bar R_1$ can be regarded as special cases of $R$ and $R_1$, when $\hat s$ sampled from $\Ncal(s,\sigma^2)$ is equal to $s$:
\begin{align}
\bar{R}(\pi) &= \underset{(s,s') \sim \Dcal}{\E} \left[\frac{\exp\left(\alpha   V\left(s'\right)\right)}{Z(s)} \log M  (s'|s, \pi(\cdot|s))\right],\\
\bar{R}_1(\pi) &= \underset{(s,s') \sim \Dcal}{\E} \left[\frac{\exp\left(\alpha   V\left(s'\right)\right)}{\exp\left(\alpha   V\left(s\right)\right)} \log M  (s'|s, \pi(\cdot|s))\right].
\end{align}

The proposed regularizer functions as follows: when the agent encounters OOD states, it drives the agent to choose actions leading to ID states, as discussed in \cref{sec:Simple and Value-aware OOD State Correction}. When the agent is in ID states, the ID state part of it comes into play. In the following, we show that it helps circumvent taking OOD actions by analyzing the maximizer of $\bar R, \bar R_1$ in tabular MDPs.

\begin{proposition}[]
\label{thm:optimal deter dynamics}
Suppose that the environment dynamics is deterministic,
then both $\bar{R}(\pi)$ and $\bar{R}_1(\pi)$ achieve their global maximum at the policy $\pi^*$, where\footnote{Here for clarity, we use the notation $M  $ with slightly different meanings in different cases: in the stochastic setting, $M  : \Scal \times \Acal \to \Delta(\Scal)$; in the deterministic setting, $M  : \Scal \times \Acal \to \Scal$.}
\begin{equation}
\pi^*(a|s) = \frac{1}{Z(s)}\exp\left(\alpha   V\left(M  (s, a) \right)\right)  \beta(a|s)
\end{equation}
The support of $\pi^*$ is within that of the behavior policy $\beta$:
\begin{equation}
    \operatorname{supp}(\pi^*(\cdot|s)) \subseteq \operatorname{supp}(\beta(\cdot|s)), ~\forall s\sim \Dcal
\end{equation}
and $\pi^*$ makes the following equation hold:
\begin{equation}
N^*(\cdot|s) = M  (\cdot| s, \pi^*(\cdot|s)), ~\forall s\sim \Dcal
\end{equation}
\end{proposition}

Under the deterministic dynamics condition, \cref{thm:optimal deter dynamics} shows that $\pi^*$ is constrained within the support of the behavior policy. Thus, our regularizer helps to keep the policy from taking OOD actions. Moreover, $\pi^*$ is able to exactly align $M  (\cdot| s, \pi^*(\cdot|s))$ with $N^*(\cdot|s)$, indicating the guidance of the agent to the high-value ID state distributions.

Furthermore, we show in \cref{thm:optimal stoc dynamics} that even under stochastic dynamics, the optimization of $\bar R$ and $\bar R_1$ still yields policies constrained within the support of $\beta$. Hence, \algo also exhibits the effect of OOD action suppression in this more general scenario.

\begin{proposition}[]
\label{thm:optimal stoc dynamics}
When the dynamics is stochastic, the maximizers of both $\bar{R}(\pi)$ and $\bar{R}_1(\pi)$ are constrained within the support of the behavior policy:
\begin{align}
    \operatorname{supp}(\pi^*(\cdot|s)) &\subseteq \operatorname{supp}(\beta(\cdot|s)),~\forall s\sim \Dcal\\
    \operatorname{supp}(\pi_1^*(\cdot|s)) &\subseteq \operatorname{supp}(\beta(\cdot|s)), ~\forall s\sim \Dcal
\end{align}
\end{proposition}

\section{Implementation Details}

\begin{wrapfigure}{R}{0.52\textwidth}
\vspace{-0.7cm}
\begin{minipage}{0.52\textwidth}
\begin{algorithm}[H]
\caption{\algo}
\label{alg:practical}
\begin{algorithmic}[1] 
\STATE Initialize dynamics model $M_\omega$, policy network $\pi_\phi$, $Q$-network $Q_\theta$, and target $Q$-network $Q_{\theta'}$
\STATE // {\bfseries Dynamics Model Training}
\FOR{each gradient step}
\STATE Update $\omega$ by minimizing $L_{M}(\omega)$ in Eq.~\eqref{eqn:mle}
\ENDFOR
\STATE // {\bfseries Policy Training}
\FOR{each gradient step}
\STATE Update $\theta$ by minimizing $L_Q(\theta)$ in Eq.~\eqref{eqn:Q loss}
\STATE Update $\phi$ by maximizing $J_{\pi}(\phi)$ in Eq.~\eqref{eq:PI}
\STATE Update target network: $\theta' \leftarrow (1-\tau){\theta'} + \tau\theta$
\ENDFOR
\end{algorithmic}
\end{algorithm}
\end{minipage}
\vspace{-0.1cm}
\end{wrapfigure}

\algo is easy to implement and we design the practical algorithm to be as simple as possible, retaining algorithmic simplicity and improving computational efficiency.

\paragraph{Dynamics model.}
We employ a deterministic dynamics model $M_\omega$. The loss for training the model is 
\begin{equation}
\label{eqn:mle}
L_M(\omega)=\underset{(s,a,s')\sim \mathcal{D}}{\mathbb{E}}[\| M_\omega(s, a)-s'\|^2_2]
\end{equation}

\paragraph{Policy improvement.}
With a deterministic model, we replace the log-likelihood in $R_1(\pi)$ with mean squared error. 
It is a common approach in RL algorithms to convert a maximum likelihood estimation problem into a regression problem when dealing with Gaussians with fixed variance~\citep{fujimoto2021minimalist}. 
As discussed in \cref{sec:Simple and Value-aware OOD State Correction}, we also adopt a deterministic policy model $\pi_\phi$. 
Thus, we have the following policy regularizer:
\begin{equation}
\begin{aligned}
\label{eq:reg mse}
R_2(\pi_\phi) = \underset{(s,s') \sim \Dcal}{\E}~ \underset{\hat s \sim \Ncal_\sigma(s)}{\E} \left[\frac{\exp\left(\alpha   V_\theta\left(s'\right)\right)}{\exp\left(\alpha   V_\theta\left(s\right)\right)} \| M_\omega(\hat s, \pi_\phi(\hat s))-s'\|^2_2 \right],
\end{aligned}
\end{equation}
where $V_\theta(s) = Q_\theta(s,\bar{\pi}_\phi(s))$ and $\bar{\pi}_\phi$ means $\pi_\phi$ with detached gradients. Using deterministic policy also simplifies the training process without learning a $V$-function.
Combining $R_2(\pi_\phi)$ with the standard policy improvement objective, we update the policy by maximizing:
\begin{equation}
\begin{aligned}
\label{eq:PI}
J_{\pi}(\phi) = (1-\lambda) \mathbb{E}_{s \sim \mathcal{D}}\left[Q_{\theta}\left(s, \pi_\phi(s)\right)\right] + \lambda R_2(\pi_\phi),
\end{aligned}
\end{equation}
where $\lambda$ is a hyperparameter to balance the two terms.
Additionally, following TD3+BC~\citep{fujimoto2021minimalist}, we also normalize $Q_{\theta}$ in the first term in each mini-batch to maintain a balanced scale across tasks.

\textbf{Overall algorithm.~~}
Putting everything together, we present our final algorithm in \cref{alg:practical}.

\section{Experiments}\label{sec:experiments}
In this section, we conduct several experiments to examine the performance and properties of \algo. 
Please refer to \cref{appsec:exp_details,appsec:Additional Experimental Results} for experimental details and additional results.

\subsection{Empirical Evidence of OOD State Correction and OOD Action Suppression}
\paragraph{OOD state correction.}
To examine the OOD state correction ability, we compare the state distributions generated by the learned policies of different algorithms with the state distribution of the offline dataset. In detail, we first train \algo, CQL~\citep{kumar2020conservative}, and TD3+BC~\citep{fujimoto2021minimalist}, and then collect 50,000 samples by running the trained policies separately. We also sample 50,000 states randomly from the offline dataset for comparison.
\cref{fig:tsne_4fig_cql,fig:tsne_4fig_td3bc,fig:tsne_4fig_scas} plot the state distributions in halfcheetah-medium-expert~\citep{fu2020d4rl} with t-SNE~\citep{van2008visualizing},
and \cref{fig:tsne_4fig_value} visualizes the optimal value of each state. We access these values from the learned value function obtained by running TD3~\citep{fujimoto2018addressing} online to convergence.

In \cref{fig:tsne_4fig_cql,fig:tsne_4fig_td3bc}, we observe that the policies learned by CQL and TD3+BC tend to produce OOD states. As depicted in \cref{fig:tsne_4fig_value}, these OOD states have extremely low values, so entering them can be detrimental to performance. 
In contrast, the state distribution induced by \algo is almost entirely within the support of the offline distribution, demonstrating the OOD state correction ability of \algo.
Moreover, we also note that in the low-value area of the offline state distribution~(the grey circle in \cref{fig:tsne_4fig_value}), \algo exhibits a very low state density, which could be attributed to \algo's value-aware OOD state correction.
We refer the reader to \cref{appsec:Additional Results on OOD State Correction} for additional experiments validating the OOD state correction effects.

\begin{wrapfigure}{r}{8cm}
    \centering
    \vspace{-0.45cm}
    \hspace{-0.1cm}
    \includegraphics[scale=0.15]{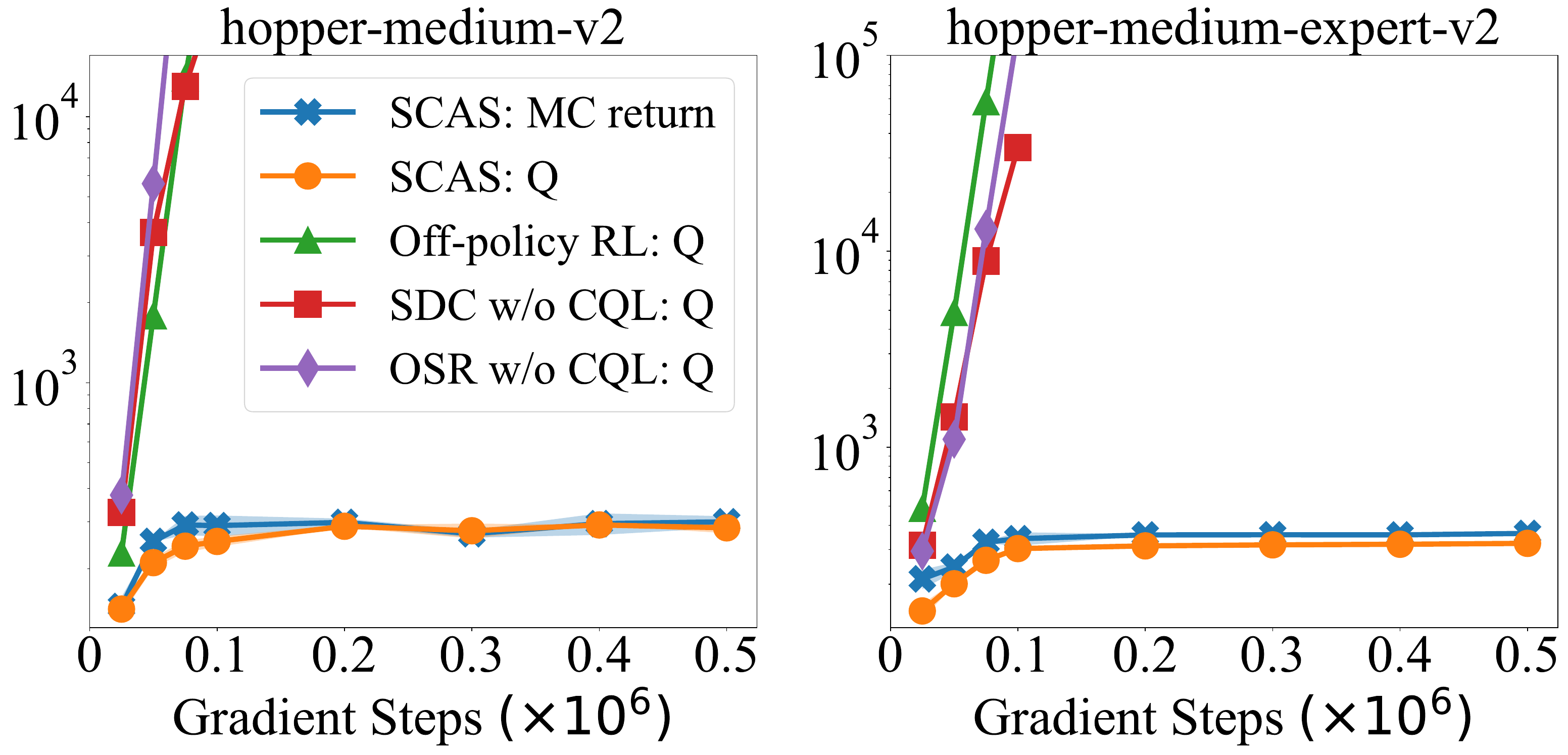}
    \vspace{-0.2cm}
    \caption{Oracle Q-values of \algo~(estimated by MC return) and learned Q-values of \algo and other algorithms across optimization steps. 
Only \algo's OOD state correction term can achieve OOD action suppression and prevent value over-estimation (divergence).}
    \label{fig:Q value}
    \vspace{-0.2cm}
\end{wrapfigure}

\paragraph{OOD action suppression.}
We empirically evaluate the OOD action suppression effects through the lens of value estimates. 
We compare \algo with three baselines: (1) ordinary off-policy RL which is \algo with $\lambda=0$~(all other implementations are the same); (2) SDC~\cite{zhang2022state} without additional CQL~\citep{kumar2020conservative} term to suppress OOD actions; (3) OSR~\cite{jiang2023recovering} without additional CQL term.
We conduct experiments on D4RL datasets~\citep{fu2020d4rl}. 
Since value over-estimation~(divergence) is the main consequence and evidence of OOD actions~\citep{fujimoto2019off}, 
we plot the learned Q-values of \algo and the baselines in \cref{fig:Q value}. We also include the oracle Q-values of \algo by rollouting the trained policy for $1,000$ episodes and evaluating the Monte-Carlo return. Additional results are provided in \cref{appsec:Additional Value Estimation Results}.

The results show that the learned Q-values of ordinary off-policy RL, SDC without CQL, and OSR without CQL diverge at early learning stages, suggesting that the algorithms suffer from severe OOD actions. By contrast, the learned Q-values of \algo stay close to the oracle Q-values. This indicates that \algo regularization alone is able to suppress OOD actions.

\subsection{Comparisons on Offline RL Benchmarks}

\begin{table}[t]
\caption{Averaged normalized scores on Gym locomotion and AntMaze tasks over five random seeds.
}\label{tab:baselines}
  \small
    \footnotesize
  \centering
\setlength{\tabcolsep}{3.0pt}
\begin{tabular}{@{}l cccccccc c@{}}
    \toprule
Dataset~(v2) & BC  & MOPO  & OneStep & TD3BC & CQL   & IQL    &OSR &SDC & \algo~(Ours) \\ \midrule
halfcheetah-med & 42.0 & \textbf{73.1} & 50.4 & 48.3 & 47.0 & 47.4  &45.1$\pm$0.8&45.9$\pm$0.5&  46.6$\pm$0.2 \\
hopper-med & 56.2 & 38.3 &  {87.5} & 59.3 & 53.0 & 66.2 &62.0$\pm$3.6 &64.7$\pm$3.5& \textbf{102.5$\pm$0.3} \\
walker2d-med & 71.0 & 41.2 &  \textbf{84.8} & \textbf{83.7} & 73.3 & 78.3   &\textbf{80.1$\pm$1.8}&\textbf{82.7$\pm$1.9}& \textbf{82.3$\pm$3.0} \\
halfcheetah-med-rep & 36.4 &\textbf{69.2} & 42.7 & 44.6 & 45.5 & 44.2 &43.3$\pm$0.2 &45.1$\pm$0.5&  44.0$\pm$0.3 \\
hopper-med-rep & 21.8 & 32.7 &  \textbf{98.5} & 60.9 & 88.7 & 94.7  &42.1$\pm$12.3&94.8$\pm$6.5&  \textbf{101.6$\pm$1.0} \\
walker2d-med-rep & 24.9 & 73.7 & 61.7 & \textbf{81.8} & \textbf{81.8} & 73.8 &\textbf{78.1$\pm$1.8}&\textbf{78.5$\pm$6.0}&  \textbf{78.1$\pm$4.5} \\
halfcheetah-med-exp & 59.6 & 70.3 & 75.1 & \textbf{90.7} & 75.6 & 86.7 &63.7$\pm$14.5 &76.3$\pm$5.2&  \textbf{91.7$\pm$2.7} \\
hopper-med-exp & 51.7 & 60.6 &  \textbf{108.6} & 98.0 &  {105.6} & 91.5 &78.9$\pm$16.4&99.9$\pm$8.5   & \textbf{109.7$\pm$3.5} \\
walker2d-med-exp & 101.2 & 77.4 &  \textbf{111.3} &  \textbf{110.1} &  \textbf{107.9} &  \textbf{109.6} &\textbf{108.1$\pm$4.4}& \textbf{109.2$\pm$1.4}& \textbf{108.4$\pm$3.7} \\
halfcheetah-rand & 2.6 & \textbf{35.9} & 2.3 & 11.0 & 17.5 & 13.1 &1.6$\pm$0.1 &14.2$\pm$0.7&  12.2$\pm$3.2 \\
hopper-rand & 4.1 & 16.7 & 5.6 & 8.5 & 7.9 & 7.9  &3.7$\pm$2.6&3.1$\pm$2.8&  \textbf{31.4$\pm$0.1} \\
walker2d-rand & 1.2 & 4.2 & \textbf{6.9} & 1.6 & 5.1 & 5.4  &-0.1$\pm$0.0&0.2$\pm$0.4& 1.4$\pm$1.1 \\
\midrule
locomotion total & 472.7 & 593.3 & 735.4 & 698.5 & 708.9 & 718.8 &606.7   & 714.6 &  \textbf{810.1} \\ 
\midrule \midrule
antmaze-umaze & 66.8 & 0.0 & 54.0 & 73.0 & 82.6 & \textbf{89.6} &87.4$\pm$5.0 &81.4$\pm$3.8& \textbf{90.4$\pm$4.3} \\
antmaze-umaze-div & 56.8 & 0.0 & 57.8 & 47.0 & 10.2 & \textbf{65.6} &55.6$\pm$8.0  &49.6$\pm$10.4&   \textbf{63.8$\pm$16.7} \\
antmaze-med-play & 0.0 & 0.0 & 0.0 & 0.0 & 59.0 &   \textbf{76.4} &22.6$\pm$7.6  &55.0$\pm$9.6&   \textbf{76.6$\pm$3.9} \\
antmaze-med-div & 0.0 & 0.0 & 0.6 & 0.2 & 46.6 &   {72.8} &19.6$\pm$5.8  &56.6$\pm$10.3&   \textbf{80.4$\pm$5.4} \\
antmaze-large-play & 0.0 & 0.0 & 0.0 & 0.0 & 16.4 & 42.0 &0.0$\pm$0.0  &20.8$\pm$8.0&   \textbf{49.0$\pm$4.0} \\
antmaze-large-div & 0.0 & 0.0 & 0.2 & 0.0 & 3.2 &   46.0 &0.0$\pm$0.0  &25.8$\pm$7.5&   \textbf{50.6$\pm$7.2} \\\midrule
antmaze total & 123.6 & 0.0 & 112.6 & 120.2 & 218   & 392.4  &185.2& 289.2&   \textbf{410.8} \\\midrule \midrule
runtime & 30m &  900m & 120m & 60m & 250m   & 100m   &300m & 420m &  140m \\\midrule
hyperparameter tuning & \textbf{w/o}  & w/ & \textbf{w/o} & \textbf{w/o} & w/   & \textbf{w/o}  &w/ &w/ & \textbf{w/o} \\ \bottomrule
\end{tabular}%
\end{table}

\textbf{Tasks.~~}
We evaluate \algo on D4RL~\cite{fu2020d4rl} and NeoRL~\cite{qin2022neorl} benchmarks. In D4RL, we conduct experiments on Gym locomotion tasks and much more challenging AntMaze tasks. Due to the space limit, \textit{the results on NeoRL} are deferred to \cref{apptab:NeoRL} in \cref{appsec:NeoRL}.

\textbf{Baselines.~~}
We compare \algo with prior state-of-the-art offline RL methods as well as the ones specifically designed for OOD state correction, including BC~\citep{pomerleau1988alvinn}, MOPO~\citep{yu2020mopo}, OneStep RL~\citep{brandfonbrener2021offline}, CQL~\citep{kumar2020conservative}, TD3+BC~\citep{fujimoto2021minimalist}, IQL~\citep{kostrikov2022offline}, SDC~\citep{zhang2022state} and OSR~\citep{jiang2023recovering}.

\textbf{Hyperparamter tuning.~~}
Offline RL methods are appealing for their ability to generate effective policies without online interaction.
Nevertheless, many existing offline RL works involve dataset-specific hyperparameter tuning. 
The reduction of hyperparameter tuning is crucial for improving practical applicability.
In this work, \algo uses \textit{a single set of hyperparameters for all datasets} in D4RL and NeoRL benchmarks to obtain the reported results.

\textbf{Comparisons with baselines.~~}
On D4RL, comparisons of performance, runtime, and hyperparameter tuning information are shown in \cref{tab:baselines}. 
We refer the reader to \cref{appsec:learning curves} for learning curve details of \algo.
On the Gym locomotion tasks, \algo outperforms prior methods on most datasets and achieves the highest total score with a single set of hyperparameters.
On the challenging AntMaze tasks, \algo performs better than IQL and outperforms other methods by a very large margin.
In NeoRL~(\cref{apptab:NeoRL}), \algo performs comparably to MOBILE~\citep{sun2023model} and outperforms other baselines.

\textbf{Runtime.~~} 
We present the runtime of algorithms at the bottom of \cref{tab:baselines}. \algo exhibits significantly lower runtime than MOPO, SDC, and OSR and is comparable to other model-free baselines.

\textbf{Generality.~~}
SCAS is a generic model-based regularizer that can be easily integrated into existing offline RL algorithms. The corresponding results and analysis are provided in \cref{appsec:combine}.

\subsection{Comparisons in Perturbed Environments}
\begin{figure}
	\centering
	\includegraphics[width=\textwidth]{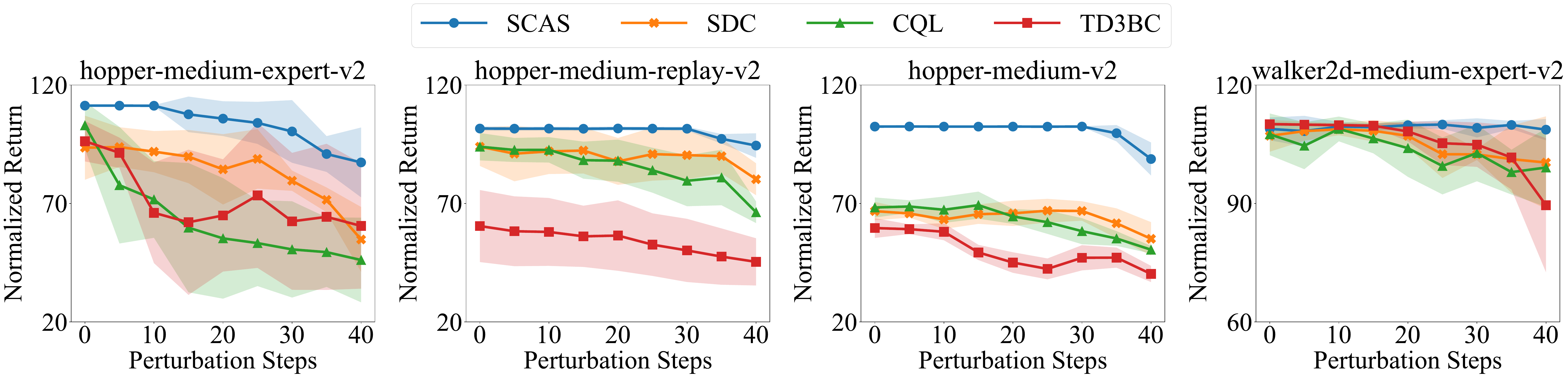}
\vspace{-4mm}
\caption{Comparisons in the perturbed environments with varying perturbation levels.
The perturbation steps are the steps of Gaussian noise added to the conducted actions in an episode.
\algo exhibits better robustness against environmental perturbations during the test phase.
}
\label{fig:perturb}
\vspace{-2mm}
\end{figure}
In this section, we evaluate the algorithms in a more real-world setting where the agent receives uncertain perturbations during test time. 
OOD state correction is even more critical in such scenarios since the agent can enter OOD states after perturbation. 
To simulate this scenario, we add varying steps of Gaussian noise with a magnitude of $0.5$ to the actions conducted by the policy during test time.
Specifically, the policy is trained on standard D4RL datasets but is tested in the perturbed environments. 
We control the strength of perturbations by adjusting the number of perturbation steps.

\cref{fig:perturb} shows the results of TD3+BC, CQL, SDC, and \algo on various datasets over five random seeds.
We observe that \algo consistently outperforms previous methods across different perturbation levels and also exhibits less performance degradation against perturbations.
Therefore, \algo enjoys better robustness against perturbations in the complex and unpredictable environments.

\subsection{Parameter Study}
\label{sec:ablation}
We examine the effects of the inverse temperature $\alpha  $, the balance coefficient $\lambda$, and the noise scale $\sigma$. Due to the space limit, \textit{the results for $\sigma$ and on additional datasets} are deferred to \cref{appsec:ablation}.
A sensitivity analysis on dynamics model errors is also provided in \cref{appsec:checkpoint}.

\begin{figure}
    \centering
    \subfigure[Effects of the inverse temperature $\alpha$.]{
    \label{fig:temp}
    \includegraphics[width=0.48\textwidth]{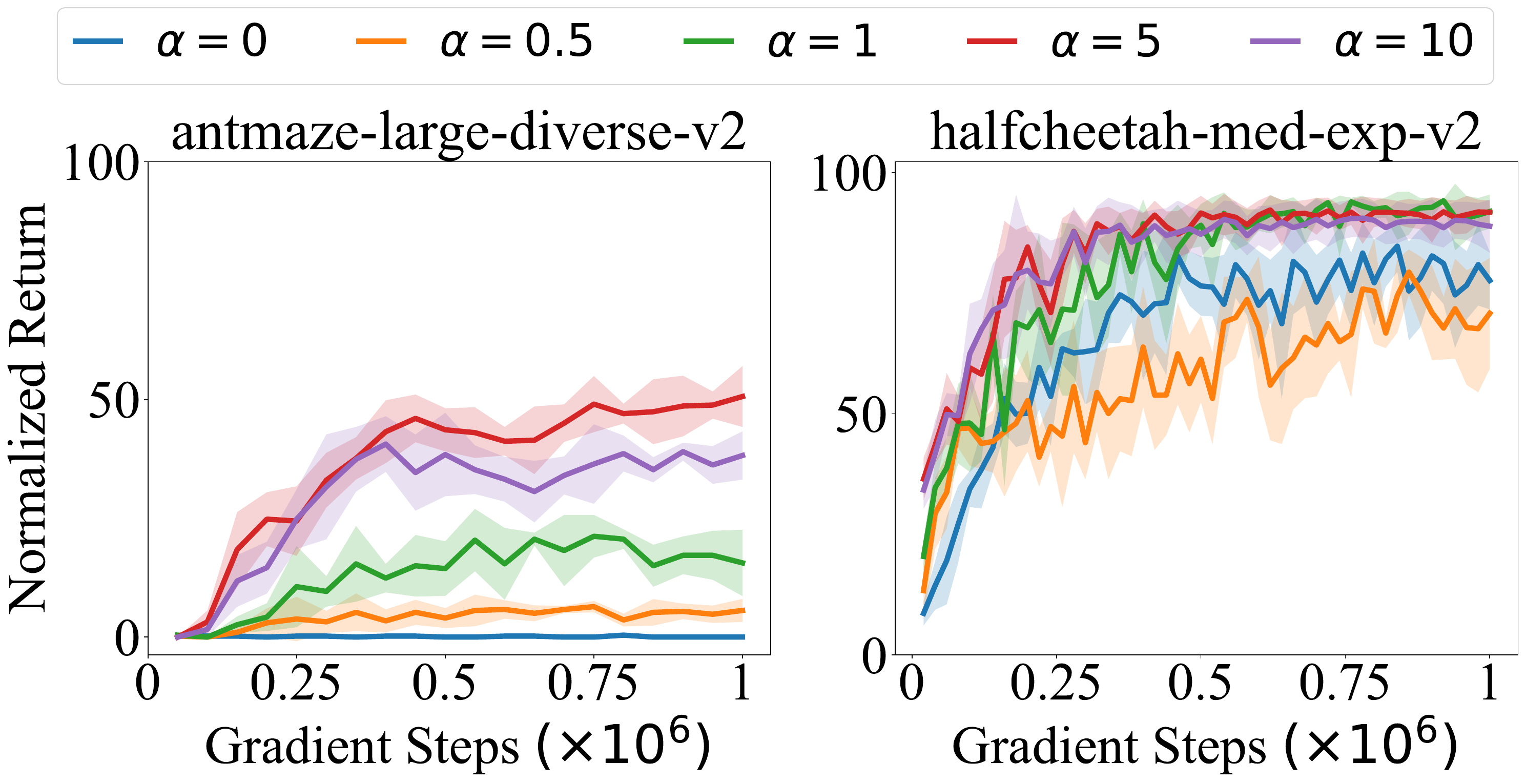}
    }\hspace{-2mm}
    \subfigure[Effects of the balance coefficient $\lambda$.]{
    \label{fig:lam}
    \includegraphics[width=0.48\textwidth]{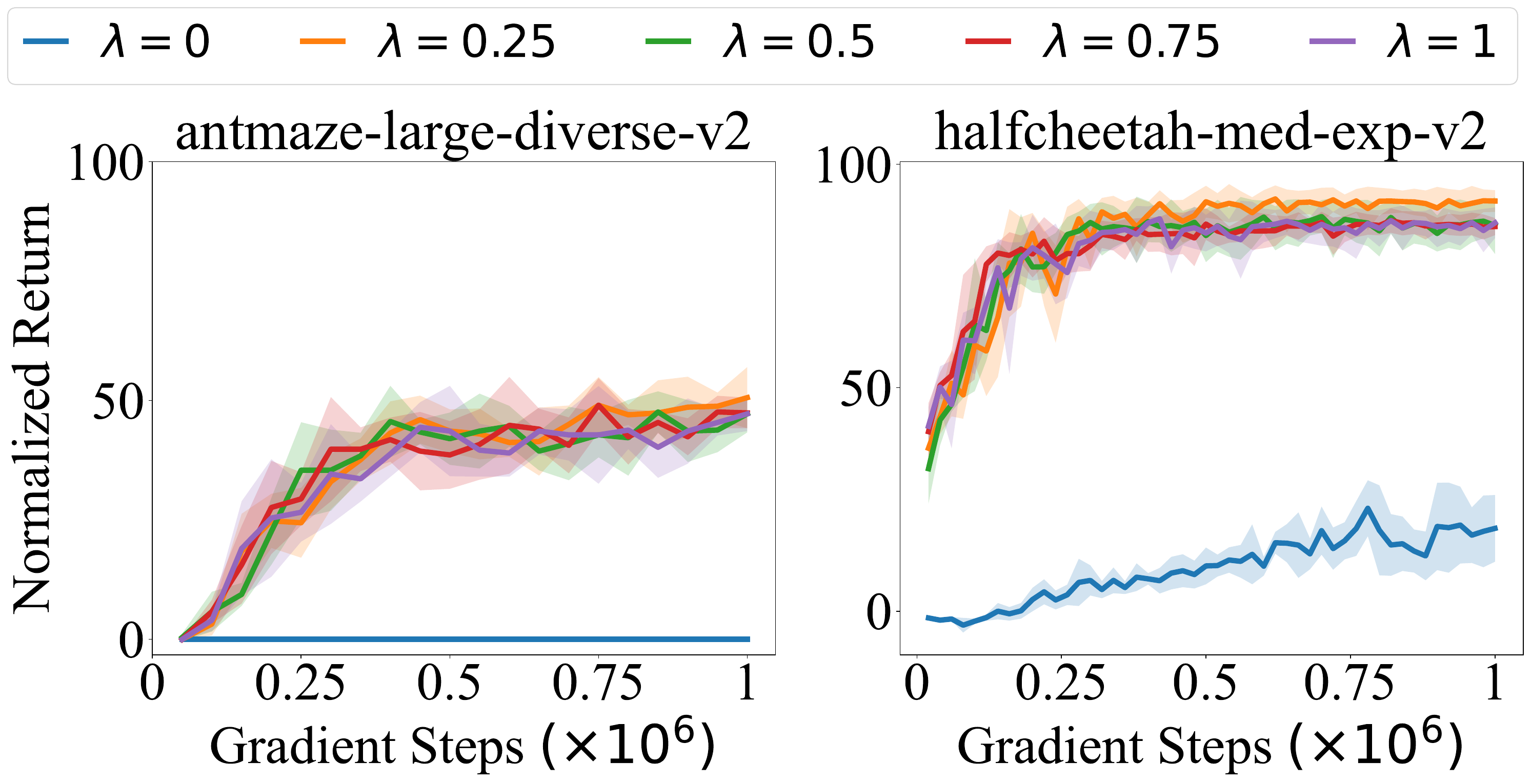}
    }\hspace{-2mm}
    \vspace{-1mm}
    \caption{Parameter study on the inverse temperature $\alpha$ and the balance coefficient $\lambda$. 
(a) An appropriately large $\alpha$ is crucial for achieving good performance.
(b) The proposed \algo regularization is essential and demonstrates robustness to changes in $\lambda$.}
    \label{fig:ablation}
    \vspace{-3mm}
\end{figure}

\textbf{Inverse temperature $\alpha  $.~~}
$\alpha  $ is the key hyperparameter in \algo for achieving value-aware OOD state correction. 
If $\alpha  =0$, the effect degenerates to vanilla OOD state correction. 
\cref{fig:temp} displays the learning curves of \algo with different $\alpha  $.
The results show that \textbf{a large $\alpha$ is \textit{crucial}} for achieving good performance~(also verified on more tasks in \cref{appfig:temp}), 
clearly demonstrating the effectiveness of our value-aware OOD state correction. 
However, too large $\alpha  $~($\alpha  =10$) induces less satisfying performance, probably due to the increased variance of the learning objective.

\textbf{Balance coefficient $\lambda$.~~}
$\lambda$ in Eq.~\eqref{eq:PI} controls the balance between vanilla policy improvement and \algo regularization. 
We vary $\lambda$ within the range $[0,1]$ and present the learning curves of \algo in \cref{fig:lam}.
Notably, \algo is able to converge to good performance over a very wide range of $\lambda$~(also verified on more tasks in \cref{appfig:lam}). 
An interesting finding is that even when $\lambda=1$ and the signal from RL improvement~(max Q) is removed, \algo still performs well on most tasks. This could be attributed to the fact that value-aware OOD state correction implies some sort of improvement in policy by maximizing the values of policy-induced next states.

\section{Conclusion and Limitations}
\label{sec:Conclusion and Limitations}
In this paper, we systematically analyze the OOD state issue in offline RL and propose \algo, a simple yet effective approach that unifies \textit{OOD state correction} and \textit{OOD action suppression}. \algo also achieves \textit{value-aware} OOD state correction, significantly improving performance over vanilla OOD state correction. Empirical results validate the properties of \algo, showcasing its superior performance on the offline RL benchmarks and its enhanced robustness in perturbed environments.

However, our work also has some limitations. For example, current \algo primarily focuses on continuous control tasks. In discrete settings, algorithmic components like state perturbation strategy would be different, which would be an interesting direction for future work. Moreover, we anticipate employing more advanced dynamics models, such as ensembles~\citep{yu2020mopo} and diffusion models~\citep{janner2022planning}, to further improve the performance of our method.

\section*{Acknowledgment}
We thank the anonymous reviewers for feedback on an early version of this paper.
This work was supported by the National Key R\&D Program of China under Grant 2018AAA0102801, National Natural Science Foundation of China under Grant 61827804.

\bibliographystyle{plainnat}
\bibliography{neurips_2024.bib}

\clearpage

\appendix

\newpage

\section{Related Work}
\label{related works}
\paragraph{Model-free offline RL.} 
In offline RL, extrapolation error and overestimation caused by OOD actions pose significant challenges.
Among model-free solutions, value regularization methods penalize the $Q$-values of OOD actions~\citep{kumar2020conservative, an2021uncertainty,kostrikov2021offline,lyu2022mildly,bai2022pessimistic,yang2022rorl,mao2023supportedv}, while policy constraint approaches compel the trained policy to be close to the behavior policy, either explicitly via divergence penalties~\citep{wu2019behavior,kumar2019stabilizing,fujimoto2021minimalist}, implicitly by weighted behavior cloning~\citep{peng2019advantage,nair2020awac,wang2020critic,mao2023supportedt}, or directly through specific parameterization of the policy~\citep{fujimoto2019off,ghasemipour2021emaq}. Relatively independently, in-sample learning methods formulate the Bellman target using only the actions in the dataset to avoid OOD actions~\citep{brandfonbrener2021offline,kostrikov2022offline,zhang2023insample,xu2023offline}.
Recently, some works aim to learn the optimal policy within the support of the dataset (known as in-support or in-sample optimal policy) in a theoretically sound way and are less affected by the average quality of the dataset~\citep{mao2023supportedt,mao2023supportedv,wu2022supported}.
However, existing popular offline RL approaches primarily focus on the OOD action issue during training and often neglect the OOD state issue during the test phase. 

\paragraph{Model-based offline RL.}
Model-based RL methods learn a model of the environment and generate synthetic data from that model to optimize the policy~\citep{sutton1991dyna,janner2019trust,kaiser2019model}. To ensure conservatism in offline RL, \citet{kidambi2020morel} and \citet{yu2020mopo} estimate the uncertainty in the model and apply reward penalties for state-action pairs with high uncertainty. 
Some model-based approaches also introduce conservatism similarly to model-free ones, employing techniques like value regularization~\citep{yu2021combo} and policy constraint~\citep{matsushima2021deploymentefficient}.
Recently, \citet{sun2023model} conducts uncertainty quantification through the inconsistency of Bellman estimations under the learned dynamics ensemble.
However, model-based methods often come with a high computational burden~\citep{janner2019trust}, and their effectiveness relies heavily on the quality of the trained model~\citep{moerland2023model}. In contrast, our algorithm leverages the dynamics model to propagate policy gradients, make one-step predictions, and regularize policy training, leading to significantly improved computational efficiency and relatively high prediction accuracy.

\paragraph{OOD state correction.}
In offline RL, OOD state correction deserves more attention as the state deviation during the test phase can accumulate over time steps, severely degrading performance~\citep{levine2020offline}. 
Existing limited solutions aim to train the policy to correct the agent from OOD states to ID states~\citep{zhang2022state,jiang2023recovering}.
Specifically, SDC~\citep{zhang2022state} builds a dynamics model and a state transition model, and aligns the policy-induced next state distributions at OOD states with the state transition model.
On the other hand, OSR~\citep{jiang2023recovering} utilizes an inverse dynamics model to constrain the policy at OOD states.
Compared with prior methods, our proposed \algo efficiently unifies OOD state correction and OOD action suppression in offline RL and additionally achieves \textit{value-aware} OOD state correction.
The DICE series of works~\citep{nachum2019algaedice,lee2021optidice,mao2024odice} share similar motivations with SCAS to some extent; however, there are significant differences between the two. Firstly, DICE is based on a linear programming framework of RL, while SCAS is based on a dynamic programming framework. Therefore, the theoretical foundations and learning paradigms of the two are inherently different. Secondly, SCAS only corrects encountered OOD states, whereas DICE algorithms require the policy-induced occupancy distribution to align with the dataset distribution. Therefore, DICE's constraints are stricter, potentially making it more susceptible to the average quality of the dataset. Lastly, theoretical and empirical evidence indicate that DICE algorithms have a problem of gradient cancellation~\citep{mao2024odice}, which imposes certain limitations on their practical effectiveness.

\section{Proofs}\label{appsec:proofs}

In this section, we present the proofs for the theories in the paper.

\subsection{Derivation of the Value-aware State Transition Distribution}
We show that Eq.~\eqref{eq:N soltion} is the optimal solution of the optimization problem Eq.~\eqref{eq:optimize N}:
\begin{equation}
\max_{N^*} ~ \underset{s \sim \Dcal}{\E} \left[\alpha   \underset{s'\sim N^*(\cdot|s)}{\E} V(s') -  \mathrm{D_{KL}}(N^*(\cdot|s) \| N  (\cdot|s))\right]
\end{equation}
We can optimize $N^*$ at each $s \in \Dcal$ separately. Thus we consider the following optimization problem:
\begin{equation}
\begin{aligned}
         \max_{\tilde{N}} ~ & \alpha   \underset{s'\sim \tilde{N}(\cdot|s)}{\E} V(s') -  \mathrm{D_{KL}}(\tilde{N}(\cdot|s) \| N  (\cdot|s)) \\
       & s.t.~     \sum_{s'} \tilde{N}(s'|s)=1,~ \forall s\in \Dcal
\end{aligned}   
\end{equation}
This constrained optimization problem is convex, and the Lagrangian is:
\begin{equation}
\begin{aligned}
\mathcal{L}(\tilde{N}) = \alpha   \underset{s'\sim \tilde{N}(\cdot|s)}{\E} V(s') - \mathrm{D_{KL}}(\tilde{N}(\cdot|s) \| N  (\cdot|s))  +\nu\left(\sum_{s'} \tilde{N}(s'|s)-1\right)
\end{aligned}
\end{equation}
The KKT condition gives:
\begin{equation}
\frac{\partial \mathcal{L}}{\partial \tilde{N}(s'|s)}=\alpha  V(s') -  \log \tilde{N}(s'|s) -1 + \log N(s'|s) + \nu =0
\end{equation}
Solving for $\tilde{N}$ gives the closed form solution $N^*$:
\begin{equation}
     N^*(s'|s) = \exp\left(\alpha   V\left(s'\right) -1 +\nu \right) N  (s'|s),~\forall s\sim\Dcal
\end{equation}
By the condition $\sum_{s'} N^*(s'|s)=1$, we can directly solve the Lagrangian multiplier $\nu$ and replace $\exp(\nu-1)$ with a normalization factor:
\begin{equation}
N^*(s'|s) = \frac{1}{Z(s)}\exp\left(\alpha   V\left(s'\right)\right) N  (s'|s),~\forall s\sim\Dcal
\end{equation}
where $Z(s)=\sum_{s'} \exp\left(\alpha   V\left(s'\right)\right) N  (s'|s)$ is the normalization factor.

\subsection{Proof of \cref{thm:optimal deter dynamics}}

\begin{proposition}[\cref{thm:optimal deter dynamics} in the main paper]
\label{appthm:optimal deter dynamics}
Suppose that the environment dynamics is deterministic,
then both $\bar{R}(\pi)$ and $\bar{R}_1(\pi)$ achieve their global maximum at the policy $\pi^*$, where\footnote{Here for clarity, we use the notation $M  $ with slightly different meanings in different cases: in the stochastic setting, $M  : \Scal \times \Acal \to \Delta(\Scal)$; in the deterministic setting, $M  : \Scal \times \Acal \to \Scal$.}
\begin{equation}
\pi^*(a|s) = \frac{1}{Z(s)}\exp\left(\alpha   V\left(M  (s, a) \right)\right)  \beta(a|s)
\end{equation}
The support of $\pi^*$ is within that of the behavior policy $\beta$:
\begin{equation}
    \operatorname{supp}(\pi^*(\cdot|s)) \subseteq \operatorname{supp}(\beta(\cdot|s)), ~\forall s\sim \Dcal
\end{equation}
and $\pi^*$ makes the following equation hold:
\begin{equation}
N^*(\cdot|s) = M  (\cdot| s, \pi^*(\cdot|s)), ~\forall s\sim \Dcal
\end{equation}
\end{proposition}

\begin{proof}
We start with $\bar{R}(\pi)$.
\begin{align}
&\argmax_\pi ~ \bar{R}(\pi)\\
=&\argmax_\pi \underset{(s,s') \sim \Dcal}{\E}~\left[\frac{1}{Z(s)}\exp\left(\alpha   V\left(s'\right)\right) \log M  (s'|s, \pi(\cdot|s))\right]\\
=&\argmax_\pi \underset{s \sim \Dcal}{\E}~\underset{s' \sim N  (\cdot|s)}{\E}~\left[\frac{1}{Z(s)}\exp\left(\alpha   V\left(s'\right)\right) \log M  (s'|s, \pi(\cdot|s))\right]\\
=&\argmax_\pi \underset{s \sim \Dcal}{\E}~\underset{s' \sim N^*(\cdot|s)}{\E}~\left[\log M  (s'|s, \pi(\cdot|s))\right]\\
=&\argmin_\pi \underset{s \sim \Dcal}{\E}~\underset{s' \sim N^*(\cdot|s)}{\E}~\left[\log N^*(s'|s) - \log M  (s'|s, \pi(\cdot|s))\right]\\
\label{appeq:reg KL}
=&\argmin_\pi \underset{s \sim \Dcal}{\E}~ \mathrm{D_{KL}}(N^*(\cdot|s) \| M  (\cdot|s, \pi(\cdot|s)))
\end{align}
The third equality holds because of the relationship between $N^*$ and $N  $ in Eq.~\eqref{eq:N soltion}:
\begin{equation}
\label{appeq:N soltion}
N^*(s'|s) = \frac{1}{Z(s)}\exp\left(\alpha   V\left(s'\right)\right) N  (s'|s),~\forall s\sim\Dcal
\end{equation}
Therefore, the maximizer of $\bar{R}(\pi)$ is equal to the solution of the minimization problem in Eq.~\eqref{appeq:reg KL}. Now consider the two distributions $N^*(\cdot|s)$ and $M  (\cdot|s, \pi(\cdot|s))$ in Eq.~\eqref{appeq:reg KL}.
\begin{align}
N^*(s'|s) &= \frac{1}{Z(s)}\exp\left(\alpha   V\left(s'\right)\right) N  (s'|s)\\
&= \frac{1}{Z(s)}\exp\left(\alpha   V\left(s'\right)\right) \sum_a \beta(a|s) M  (s'|s,a)
\end{align}
For analytical clarity, we use the notation $M  $ with slightly different meanings in different cases: in the stochastic setting, $M  : \Scal \times \Acal \to \Delta(\Scal)$; in the deterministic setting, $M  : \Scal \times \Acal \to \Scal$. With the deterministic dynamics assumption,
\begin{align}
N^*(s'|s) &= \frac{1}{Z(s)}\exp\left(\alpha   V\left(s'\right)\right) \sum_a \beta(a|s) \mathbb{I}\left[M  (s,a)=s'\right]\\
&= \sum_a \frac{1}{Z(s)}\exp\left(\alpha   V\left(s'\right)\right) \beta(a|s) \mathbb{I}\left[M  (s,a)=s'\right]\\
\label{appeq:N*}
&= \sum_a \frac{1}{Z(s)}\exp\left(\alpha   V\left(M  (s,a)\right)\right) \beta(a|s) \mathbb{I}\left[M  (s,a)=s'\right]
\end{align}
On the other hand,
\begin{align}
M  (s'|s, \pi(\cdot|s)) &= \sum_a M  (s'|s, a) \pi(a|s)\\
\label{appeq:MD1}
&= \sum_a \pi(a|s) \mathbb{I}\left[M  (s,a)=s'\right]
\end{align}
Now we define $\pi^*(a|s)$ as
\begin{equation}
\label{appeq:pi*}
\pi^*(a|s) := \frac{1}{Z(s)}\exp\left(\alpha   V\left(M  (s,a)\right)\right) \beta(a|s)
\end{equation}
We first show that $\pi^*$ is a valid policy, that is, $\pi^*$ is normalized.
\begin{align}
\pi^*(a|s) &= \frac{1}{Z(s)}\exp\left(\alpha   V\left(M  (s,a)\right)\right) \beta(a|s)\\
&= \frac{\exp\left(\alpha   V\left(M  (s,a)\right)\right) \beta(a|s)}{\sum_{s'} \exp\left(\alpha   V\left(s'\right)\right) N  (s'|s)}\\
&= \frac{\exp\left(\alpha   V\left(M  (s,a)\right)\right) \beta(a|s)}{\sum_{s'} \exp\left(\alpha   V\left(s'\right)\right) \sum_a \beta(a|s) M  (s'|s,a)}\\
&= \frac{\exp\left(\alpha   V\left(M  (s,a)\right)\right) \beta(a|s)}{\sum_a \sum_{s'} \exp\left(\alpha   V\left(s'\right)\right) \beta(a|s) M  (s'|s,a)}\\
&= \frac{\exp\left(\alpha   V\left(M  (s,a)\right)\right) \beta(a|s)}{\sum_a \exp\left(\alpha   V\left(M  (s,a)\right)\right) \beta(a|s)}
\end{align}
Therefore, $\sum_a \pi^*(a|s) = 1$.

Substitute Eq.~\eqref{appeq:pi*} into Eq.~\eqref{appeq:N*},
\begin{equation}
\label{appeq:N*1}
N^*(s'|s) = \sum_a \pi^*(a|s) \mathbb{I}\left[M  (s,a)=s'\right]
\end{equation}
Comparing Eq.~\eqref{appeq:MD1} with Eq.~\eqref{appeq:N*1}, it holds that $N^*(s'|s) = M  (s'|s, \pi^*(\cdot|s)), \forall s\sim\Dcal$. As a result, 
\begin{equation}
\underset{s \sim \Dcal}{\E}~ \mathrm{D_{KL}}(N^*(\cdot|s) \| M  (\cdot|s, \pi^*(\cdot|s))) = 0
\end{equation}
Considering the non-negativity of KL divergence, the optimization problem in Eq.~\eqref{appeq:reg KL} achieves its global minimum at $\pi^*$. Therefore, $\bar{R}(\pi)$ also achieves its global maximum at $\pi^*$.

Now we consider $\bar{R}_1(\pi)$. 
\begin{align}
&\argmax_\pi ~ \bar{R}_1(\pi)\\
=&\argmax_\pi \underset{(s,s') \sim \Dcal}{\E}~\left[\exp\left(\alpha   \left(V\left(s'\right)- V\left(s\right)\right)\right) \log M  (s'|s, \pi(\cdot|s))\right]\\
=&\argmax_\pi \underset{(s,s') \sim \Dcal}{\E}~\left[\frac{Z(s)}{\exp\left(\alpha   V\left(s\right)\right) Z(s)}\exp\left(\alpha   V\left(s'\right)\right) \log M  (s'|s, \pi(\cdot|s))\right]\\
=&\argmax_\pi \underset{s \sim \Dcal}{\E}~\underset{s' \sim N  (\cdot|s)}{\E}~\left[\frac{Z(s)}{\exp\left(\alpha   V\left(s\right)\right) Z(s)}\exp\left(\alpha   V\left(s'\right)\right) \log M  (s'|s, \pi(\cdot|s))\right]\\
=&\argmax_\pi \underset{s \sim \Dcal}{\E}~\underset{s' \sim N^*(\cdot|s)}{\E}~\left[\frac{Z(s)}{\exp\left(\alpha   V\left(s\right)\right)}\log M  (s'|s, \pi(\cdot|s))\right]\\
=&\argmin_\pi \underset{s \sim \Dcal}{\E}~\underset{s' \sim N^*(\cdot|s)}{\E}~\left[\frac{Z(s)}{\exp\left(\alpha   V\left(s\right)\right)} \left(\log N^*(s'|s) - \log M  (s'|s, \pi(\cdot|s))\right)\right]\\
\label{appeq:reg KL R1}
=&\argmin_\pi \underset{s \sim \Dcal}{\E}~ \left[\frac{Z(s)}{\exp\left(\alpha   V\left(s\right)\right)} \mathrm{D_{KL}}(N^*(\cdot|s) \| M  (\cdot|s, \pi(\cdot|s)))\right]
\end{align}
The fourth equality holds because of the relationship between $N^*$ and $N  $ in Eq.~\eqref{appeq:N soltion}.

As shown above, it holds that $N^*(s'|s) = M  (s'|s, \pi^*(\cdot|s)), \forall s\sim\Dcal$. As a result, 
\begin{equation}
\underset{s \sim \Dcal}{\E}~ \left[\frac{Z(s)}{\exp\left(\alpha   V\left(s\right)\right)} \mathrm{D_{KL}}(N^*(\cdot|s) \| M  (\cdot|s, \pi(\cdot|s)))\right]=0
\end{equation}
Considering $Z(s)/{\exp\left(\alpha   V\left(s\right)\right)}>0$ and the non-negativity of KL divergence, the optimization problem in Eq.~\eqref{appeq:reg KL R1} achieves its global minimum at $\pi^*$. Therefore, $\bar{R}_1(\pi)$ also achieves its global maximum at $\pi^*$.

In conclusion, when the environment dynamics is deterministic, both $\bar{R}(\pi)$ and $\bar{R}_1(\pi)$ achieve their global maximum at the policy $\pi^*$, and $\pi^*$ makes the following equation hold:
\begin{equation}
N^*(\cdot|s) = M  (\cdot| s, \pi^*(\cdot|s)), ~\forall s\sim \Dcal
\end{equation}
Moreover, because $\pi^*(a|s) = \frac{1}{Z(s)}\exp\left(\alpha   V\left(M  (s, a) \right)\right)  \beta(a|s)$, the support of $\pi^*$ is included by that of the behavior policy $\beta$:
\begin{equation}
    \operatorname{supp}(\pi^*(\cdot|s)) \subseteq \operatorname{supp}(\beta(\cdot|s)), ~\forall s\sim \Dcal
\end{equation}
\end{proof}

\subsection{Proof of \cref{thm:optimal stoc dynamics}}
\begin{proposition}[\cref{thm:optimal stoc dynamics} in the main paper]
\label{appthm:optimal stoc dynamics}
When the dynamics is stochastic, the maximizers of both $\bar{R}(\pi)$ and $\bar{R}_1(\pi)$ are constrained within the support of the behavior policy:
\begin{align}
    \operatorname{supp}(\pi^*(\cdot|s)) &\subseteq \operatorname{supp}(\beta(\cdot|s)),~\forall s\sim \Dcal\\
    \operatorname{supp}(\pi_1^*(\cdot|s)) &\subseteq \operatorname{supp}(\beta(\cdot|s)), ~\forall s\sim \Dcal
\end{align}
\end{proposition}

\begin{proof}
We start with $\bar{R}(\pi)$.
\begin{align}
\bar{R}(\pi):&= \underset{(s,s') \sim \Dcal}{\E}~\left[\frac{1}{Z(s)}\exp\left(\alpha   V\left(s'\right)\right) \log M  (s'|s, \pi(\cdot|s))\right]\\
&=\underset{(s,s') \sim \Dcal}{\E}~\left[\frac{1}{Z(s)}\exp\left(\alpha   V\left(s'\right)\right) \log \left(\sum_a M  (s'|s, a) \pi(a|s)\right)\right]
\end{align}
Let $\pi$ denote any valid policy. For $\forall s \in \Dcal$, define $\epsilon(s)$ and $n(s)$ as follows:
\begin{align}
\label{appeq:epsilon s}
\epsilon(s)&:=\sum_a \mathbb{I}[\beta(a|s)=0]\pi(a|s)\\
\label{appeq:n s}
n(s)&:=\sum_a \mathbb{I}[\beta(a|s)>0]
\end{align}
For $\forall s \in \Dcal$, there exists at least one action $a$ such that $(s,a) \in \Dcal$. Thus it holds that $n(s)>0,\forall s \in \Dcal$.
Then for $\forall s \in \Dcal, \forall \pi$, define $\pi_\text{in}$ as follows:
\begin{equation}
\label{appeq:pi^*'}
    \pi_\text{in}(a|s) = \left\{  
         \begin{array}{lr}  
         \pi(a|s) + \frac{\epsilon(s)}{n(s)},  & \beta(a|s)>0,  \\  
         0, & \beta(a|s)=0. \\
         \end{array}
\right.  
\end{equation}
$\pi_\text{in}$ can be seen as a projection of $\pi$ onto $\beta$'s support. Besides, for $\forall s \in \Dcal$,
\begin{align}
\sum_a \pi_\text{in}(a|s) &= \sum_a \mathbb{I}[\beta(a|s)>0] \left( \pi(a|s) + \frac{\epsilon(s)}{n(s)} \right)\\
&= \sum_a \mathbb{I}[\beta(a|s)>0] \pi(a|s) + \epsilon(s)\\
&= \sum_a \mathbb{I}[\beta(a|s)>0] \pi(a|s) + \sum_a \mathbb{I}[\beta(a|s)=0]\pi(a|s)\\
&= \sum_a \pi(a|s)\\
&= 1
\end{align}
Thus $\pi_\text{in}$ is a valid policy.

Now we compare $\bar{R}(\pi_\text{in})$ with $\bar{R}(\pi)$. For $\forall (s,s') \in \Dcal$,
\begin{align}
&\sum_a M  (s'|s, a) \pi_\text{in}(a|s) - \sum_a M  (s'|s, a) \pi(a|s)\\
= &\sum_a M  (s'|s, a) \left(\pi_\text{in}(a|s) - \pi(a|s) \right)\\
= &\sum_{\{a|\beta(a|s)>0\}} M  (s'|s, a) \left(\pi_\text{in}(a|s) - \pi(a|s) \right)\\
= &\sum_{\{a|\beta(a|s)>0\}} M  (s'|s, a) \frac{\epsilon(s)}{n(s)}\\
\geq & \quad 0
\end{align}
The second equality holds because, in tabular MDPs, the empirical dynamics model $M  $ exactly computes the conditional distribution observed in the dataset.
For transitions not contained in the dataset, $M  =0$~\citep{fujimoto2019off}.
The final inequality holds because $\epsilon(s) \geq 0$.

Therefore, 
\begin{align}
&\bar{R}(\pi_\text{in}) - \bar{R}(\pi)\\
=& \underset{(s,s') \sim \Dcal}{\E}~\left[\frac{1}{Z(s)}\exp\left(\alpha   V\left(s'\right)\right) \log \left(\frac{\sum_a M  (s'|s, a) \pi_\text{in}(a|s)}{\sum_a M  (s'|s, a) \pi(a|s)} \right)\right]\\
\geq& \underset{(s,s') \sim \Dcal}{\E}~\left[\frac{1}{Z(s)}\exp\left(\alpha   V\left(s'\right)\right) \log \left(1 \right)\right]\\
\geq& \quad 0
\end{align}
Now suppose $\pi$ is not constrained within the support of the behavior policy at some state $s_1 \in \Dcal$: $\operatorname{supp}(\pi(\cdot|s_1)) \nsubseteq \operatorname{supp}(\beta(\cdot|s_1))$. That is, $\exists \tilde{a}_1$ such that $\pi(\tilde{a}_1|s_1)>0$ and $\beta(\tilde{a}_1|s_1)=0$. Thus it holds that $\epsilon(s_1)=\sum_a \mathbb{I}[\beta(a|s_1)=0]\pi(a|s_1)>0$.
On the other hand, since $s_1 \in \Dcal$, there exists at least one action $a_1$ and one state $s_1'$ such that $(s_1,a_1,s_1') \in \Dcal$. Thus it holds that $\beta(a_1|s_1)>0$ and $M  (s_1'|s_1, a_1)>0$. As a result, 
\begin{align}
&\sum_a M  (s_1'|s_1, a) \pi_\text{in}(a|s_1) - \sum_a M  (s_1'|s_1, a) \pi(a|s_1)\\
= &\sum_{\{a|\beta(a|s_1)>0\}} M  (s_1'|s_1, a) \frac{\epsilon(s_1)}{n(s_1)}\\
> & \quad 0
\end{align}
In such case, $\bar{R}(\pi_\text{in}) > \bar{R}(\pi)$.
Therefore, if $\pi$ is not constrained within the support of the behavior policy at some state $s_1 \in \Dcal$, we can find another policy $\pi_\text{in}$ that is constrained within the support of the behavior policy and achieves higher objective function $\bar{R}(\pi_\text{in})$. 
Consequently, $\bar{R}(\pi)$ must achieve its maximum at support constrained policy $\pi^*$: $\operatorname{supp}(\pi^*(\cdot|s)) \subseteq \operatorname{supp}(\beta(\cdot|s)),~\forall s\sim \Dcal$.

Now we consider $\bar{R}_1(\pi)$. 
\begin{align}
\bar{R}_1(\pi):&= \underset{(s,s') \sim \Dcal}{\E}~\left[\exp\left(\alpha   \left(V\left(s'\right)- V\left(s\right)\right)\right) \log M  (s'|s, \pi(\cdot|s))\right]\\
&=\underset{(s,s') \sim \Dcal}{\E}~\left[\exp\left(\alpha   \left(V\left(s'\right)- V\left(s\right)\right)\right) \log \left(\sum_a M  (s'|s, a) \pi(a|s)\right)\right]
\end{align}

With the same definition of $\epsilon(s)$, $n(s)$ and $\pi_\text{in}$ as in Eq.~\eqref{appeq:epsilon s}, Eq.~\eqref{appeq:n s} and Eq.~\eqref{appeq:pi^*'}, it also holds that 
\begin{align}
&\bar{R}_1(\pi_\text{in}) - \bar{R}_1(\pi)\\
=& \underset{(s,s') \sim \Dcal}{\E}~\left[\exp\left(\alpha   \left(V\left(s'\right)- V\left(s\right)\right)\right) \log \left(\frac{\sum_a M  (s'|s, a) \pi_\text{in}(a|s)}{\sum_a M  (s'|s, a) \pi(a|s)} \right)\right]\\
\geq& \underset{(s,s') \sim \Dcal}{\E}~\left[\exp\left(\alpha   \left(V\left(s'\right)- V\left(s\right)\right)\right) \log \left(1 \right)\right]\\
\geq& \quad 0
\end{align}
As before, when supposing $\pi$ is not constrained within the support of the behavior policy at some state $s_1 \in \Dcal$, it holds that $\bar{R}_1(\pi_\text{in}) > \bar{R}_1(\pi)$. Therefore, $\bar{R}_1(\pi)$ must achieve its maximum at support constrained policy $\pi_1^*$: $\operatorname{supp}(\pi_1^*(\cdot|s)) \subseteq \operatorname{supp}(\beta(\cdot|s)),~\forall s\sim \Dcal$.

In conclusion, when the environment dynamics is stochastic, the maximizers of both $\bar{R}(\pi)$ and $\bar{R}_1(\pi)$ are constrained within the support of the behavior policy:
\begin{equation}
    \operatorname{supp}(\pi^*(\cdot|s)) \subseteq \operatorname{supp}(\beta(\cdot|s)),~ \operatorname{supp}(\pi_1^*(\cdot|s)) \subseteq \operatorname{supp}(\beta(\cdot|s)),~\forall s\sim \Dcal
\end{equation}
\end{proof}

\section{Further Discussions}

\subsection{Rationale for Choosing $\exp(\alpha V(s))$ as the Empirical Normalizer}
\label{appsec:normalizer}
Firstly, choosing $\exp(\alpha V(s))$ is intended to obtain something similar to the advantage function. With this normalizer, the weight of our regularizer is $\exp(\alpha (V(s') - V(s)))$, which is comparable to the weight $\exp(\alpha A(s,a))$ in Advantage Weighted Regression (AWR)~\citep{peng2019advantage}. Here, $V(s') - V(s)$ represents the relative advantage of the next state $s'$ compared to the current state $s$, while $A(s,a)$ reflects the relative advantage of taking action $a$ in $s$ compared to following the current policy. Comparison of the objectives of SCAS and AWR:
\begin{equation}
\text{SCAS:}\quad \exp(\alpha (V(s')-V(s))) \log(M(s'|\hat s,\pi(\hat s)))
\end{equation}
\begin{equation}
\text{AWR:}\quad \exp(\alpha A(s,a)) \log\pi(a|s)
\end{equation}

Secondly, as discussed in the paper, introducing any normalizer that depends only on $s$ (independent of $s'$) does not affect the development and analysis of our method; it is merely for computational stability. In AWR-based methods, there also exists a normalizer $Z(s)$ and they usually disregard it~\citep{peng2019advantage,nair2020awac}. The rationale behind this is similar.

\subsection{Pessimism and Robustness in SCAS}
In a specific sense, SCAS, which unifies OOD state correction and OOD action suppression, also integrates pessimism and state robustness. (1) Regarding pessimism: The OOD action suppression effect of SCAS aligns with the pessimism commonly discussed in offline RL work (being pessimistic about OOD actions)~\citep{kumar2020conservative,xie2021bellman,kumar2019stabilizing,bai2022pessimistic,shao2023counterfactual}. Unlike traditional policy constraint methods~\citep{wu2019behavior,kumar2019stabilizing,fujimoto2021minimalist,peng2019advantage}, our approach does not require the training policy to align with the behavior policy; it only requires the successor states to be within the dataset support, which is a more relaxed constraint. (2) Regarding state robustness: The OOD state correction effect of SCAS is aimed at improving the agent's robustness to OOD states during the test phase. Compared with previous works, SCAS unifies OOD state correction and OOD action suppression and additionally achieves value-aware OOD state correction. Some offline RL literature on state robustness differs from our approach; they typically consider noisy observations~\citep{yang2022rorl}, such as sensor errors. In contrast, SCAS addresses state robustness concerning actual OOD states encountered during test time, rather than noisy observations.

\subsection{Regularization Effect at ID States}
In SCAS, there is regularization on the policy's output actions at ID states. In our regularizer, the perturbed states $\hat{s}$ are sampled from $\mathcal{N}(s,\sigma^2)$, and a large portion of $\hat{s}$ will fall near the original ID state $s$ or even be approximately equal to $s$. Therefore, the policy's output actions at ID states are also regularized. For this part of the regularization, its role is equivalent to the ID state regularizer analyzed in \cref{sec:OOD Action Suppression}, which has been theoretically shown to have the effect of suppressing OOD actions. Moreover, the experimental results in \cref{sec:experiments} also demonstrate that our OOD state correction regularizer addresses the traditional issue with OOD actions.

\subsection{Differences between the OOD Action Issue and the OOD State Issue}

We further elucidate the differences between the well-known OOD action issue and the OOD state issue we analyzed. Most offline RL works focus on the OOD action issue in the training phase. That is, the trained policy outputs OOD actions to compute the target Q, which results in extrapolation error and value divergence during training~\citep{fujimoto2019off}. In contrast, the OOD state issue we defined and analyzed is in the test phase. That is, the agent can enter states out of the offline dataset during test, potentially resulting in catastrophic failure.

\section{Experimental Details}
\label{appsec:exp_details}

\begin{table}[htbp]
\caption{Hyperparameters in \algo.}

\label{apptab:hyper}
\begin{center}
\begin{tabular}{cll}
\toprule
                              & Hyperparameter          & \multicolumn{1}{l}{Value}           \\ \midrule
\multirow{12}{*}{Policy training}         & Optimizer               & \multicolumn{1}{l}{Adam ~\citep{kingma2014adam}}            \\
                              & Critic learning rate    & \multicolumn{1}{l}{$3\times 10^{-4}$}            \\
                              & Actor learning rate     & \multicolumn{1}{l}{$2\times 10^{-4}$ with cosine schedule}  \\
                              & Batch size              & 256                                 \\
                              & Discount factor                & 0.99                                \\
                              & Gradient Steps    & $10^6$                             \\
                              & Target network update rate      & 0.005                               \\
                              & Policy update frequency & 2                                   \\
                              & Number of Critics & 4                                   \\
                              & Inverse temperature $\alpha  $   & 5 \\
                              & Balance coefficient $\lambda$  & 0.25              \\
                              & Noise scale $\sigma$  & 0.003              \\
                              \midrule
\multirow{4}{*}{Dynamics training}         & Optimizer               & \multicolumn{1}{l}{Adam}            \\
                              & Learning rate    & \multicolumn{1}{l}{$1\times 10^{-3}$}            \\
                              & Batch size              & 256                                 \\
                              & Gradient Steps    & $5 \times 10^5$                             \\
                              \midrule
\multirow{3}{*}{Architecture} & Actor    & input-256-256-output                                 \\
                              & Critic & input-256-256-1                                      \\
                              & Dynamics & input-256-256-256-256-output                                      \\
                              \bottomrule
\end{tabular}
\end{center}
\end{table}

All hyperparameters of \algo are included in \cref{apptab:hyper}. 
Note that we use this same set of hyperparameters to obtain all the results reported in this paper~(except for parameter study). Following TD3+BC~\cite{fujimoto2021minimalist}, we normalize the states in all datasets except for antmaze-large. We clip the exponentiated weight $\exp\left(\alpha   V_\theta\left(s'\right)-\alpha   V_\theta\left(s\right)\right)$ in Eq.~\eqref{eq:reg mse} to $(-\infty, 50]$. Following the suggestions in the benchmark~\cite{fu2020d4rl}, we subtract 1 from the rewards for the Antmaze datasets.

Our evaluation criteria follow those used in most previous works. For the Gym locomotion tasks, we average returns over 10 evaluation trajectories and 5 random seeds, while for the Ant Maze tasks, we average over 100 evaluation trajectories and 5 random seeds. The reported results are the normalized scores, which are offered by the D4RL benchmark~\cite{fu2020d4rl} to measure how the learned policy compared with random and expert policy:
\begin{equation*}
\text{D4RL score} = 100 \times \frac{\text{learned policy return}-\text{random policy return}}{\text{expert policy return}-\text{random policy return}}
\end{equation*}

The results of baselines reported in \cref{tab:baselines} are obtained as follows. We re-run OSR~\citep{jiang2023recovering} on all datasets using their official codebase\footnote{\url{https://github.com/Jack10843/OSR}} and tune the hyperparameters for each dataset as specified in their paper. We implement SDC~\citep{zhang2022state} and re-run it on all datasets. We use the SDC-related hyperparameters as specified in their paper, and sweep the CQL-related hyperparameters in \{1,2,5,10,20\} for each dataset. We re-run OneStep RL~\citep{brandfonbrener2021offline} on all datasets using their official codebase\footnote{\url{https://github.com/davidbrandfonbrener/onestep-rl}} and the default hyperparameters. We implement BC~\citep{pomerleau1988alvinn} based on the TD3+BC repository\footnote{\url{https://github.com/sfujim/TD3_BC}} and re-run it on all datasets. The results of other baselines are taken from \citep{bai2022pessimistic} and \citep{wu2022supported}.
The runtime in \cref{tab:baselines} is obtained by running offline RL algorithms on halfcheetah-medium-replay-v2 on a GeForce RTX 3090.

\cref{fig:tsne_4fig_cql,fig:tsne_4fig_td3bc,fig:tsne_4fig_scas,fig:tsne_4fig_value} share the same embedding function obtained by running t-SNE on the set of all 200,000 samples (50,000 samples each from the dataset, CQL, TD3+BC, and SCAS). This ensures a clear visual comparison. \cref{fig:tsne_4fig_value} contains all the 200,000 samples, which is the union of the points in \cref{fig:tsne_4fig_cql,fig:tsne_4fig_td3bc,fig:tsne_4fig_scas}.

\section{Additional Experimental Results}
\label{appsec:Additional Experimental Results}

\subsection{Additional Value Estimation Results}
\label{appsec:Additional Value Estimation Results}

Under the same setting of \cref{fig:Q value}, we conduct experiments on the additional datasets. The results are shown in \cref{appfig:Q value}. We omit the Q values of Off-policy RL, SDC w/o CQL, and OSR w/o CQL at higher numbers of optimization steps, because these Q values diverge in the early learning stage, and plotting their Q values at later optimization steps would result in an excessive range on the vertical axis.
The additional results also show that only SCAS's OOD state correction term can achieve OOD action suppression and prevent value over-estimation. 

\begin{figure}[ht]
	\centering
	\includegraphics[width=0.65\textwidth]{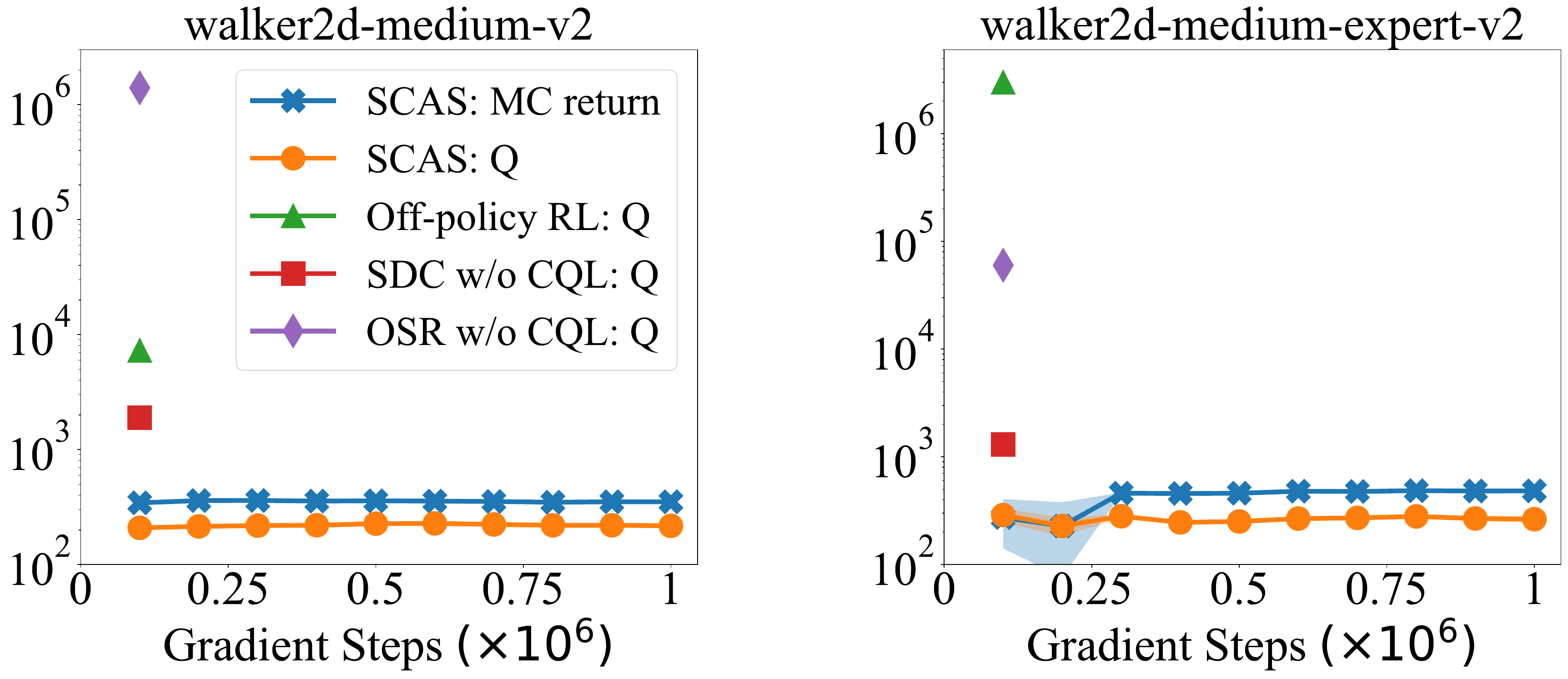}
\caption{
Oracle Q-values of \algo~(estimated by MC return) and learned Q-values of \algo and other algorithms across optimization steps. 
Here Off-policy RL is \algo with weight $\lambda = 0$ in Eq.~\eqref{eq:PI}. 
Only \algo's OOD state correction term can achieve OOD action suppression and prevent value over-estimation (divergence).
}
\label{appfig:Q value}
\end{figure}

\subsection{Additional Results on OOD State Correction}
\label{appsec:Additional Results on OOD State Correction}
To further examine the OOD state correction effects of \algo, we conduct experiments on a modified D4RL maze2d-open-v0~\citep{fu2020d4rl}. It is a 2D point robot navigation task in a rectangle map with vertices $(0,0)$ and $(3,5)$. The agent needs to reach the goal at $(2,3)$. We modify the dataset by removing all the transitions containing states in a rectangle with vertices $(0,0)$ and $(1.5,2.5)$. During test, we let the initial state be randomly distributed in this OOD region. We train algorithms over $10^6$ gradient steps and average returns over 1000 evaluation trajectories. 

The results of BC~\citep{pomerleau1988alvinn}, TD3+BC~\citep{fujimoto2021minimalist}, CQL~\citep{kumar2020conservative}, MOPO~\citep{yu2020mopo}, IQL~\citep{kostrikov2022offline}, and \algo are reported in \cref{apptab:maze2d}.
With the OOD state correction signals, SCAS corrects the agent out of the OOD region more quickly and stably, achieving significantly better performance than typical offline RL methods.

\begin{table}[h]
\caption{Comparisons in modified maze2d-open-v0 over five random seeds.
}\label{apptab:maze2d}
  \centering
\setlength{\tabcolsep}{3.0pt}
\begin{tabular}{@{}l cccccc c@{}}
    \toprule
     & BC & TD3+BC & CQL & MOPO & IQL & SCAS \\ \midrule
    Steps out of OOD & 84.7$\pm$44.7 & 58.0$\pm$35.7 & 63.8$\pm$33.0 & 50.6$\pm$25.4 & 37.7$\pm$6.7 & \textbf{22.8$\pm$3.1} \\ 
    D4RL score & 38.5$\pm$25.4 & 63.9$\pm$39.3 & 41.2$\pm$42.0 & 110.1$\pm$78.8 & 335.0$\pm$114.9 & \textbf{571.9$\pm$2.7} \\ 
\bottomrule
\end{tabular}%
\end{table}

\subsection{Comparisons on the NeoRL Benchmark}
\label{appsec:NeoRL}

\begin{table}[h]
\caption{Averaged normalized scores on the NeoRL benchmark over four random seeds.
}\label{apptab:NeoRL}
  \centering
\setlength{\tabcolsep}{3.0pt}
\begin{tabular}{@{}l cccccc c@{}}
    \toprule
& BC & TD3BC & CQL & EDAC & MOPO & MOBILE & SCAS \\ \midrule
Hopper-High & 43.1 & 75.3 & 76.6 & 52.5 & 11.5 & 87.8 & \textbf{100.5$\pm$7.8} \\ 
Hopper-Med & 51.3 & 70.3 & 64.5 & 44.9 & 1.0 & 51.1 & \textbf{94.6$\pm$9.3} \\ 
Hopper-Low & 15.1 & 15.8 & 16.0 & 18.3 & 6.2 & 17.4 & \textbf{19.7$\pm$1.2} \\ 
Walker2d-High & 72.6 & 69.6 & \textbf{75.3} & \textbf{75.5} & 18.0 & \textbf{74.9} & \textbf{74.6$\pm$0.7} \\ 
Walker2d-Med & 48.7 & 58.5 & 57.3 & 57.6 & 39.9 & 62.2 & \textbf{63.4$\pm$1.0} \\ 
Walker2d-Low & 28.5 & 43.0 & \textbf{44.7} & 40.2 & 11.6 & 37.6 & 34.4$\pm$1.3 \\ 
HalfCheetah-High & 71.3 & 75.3 & 77.4 & 81.4 & 65.9 & \textbf{83.0} & 77.0$\pm$0.5 \\ 
HalfCheetah-Med & 49.0 & 52.3 & 54.6 & 54.9 & 62.3 & \textbf{77.8} & 53.1$\pm$0.1 \\ 
HalfCheetah-Low & 29.1 & 30.0 & 38.2 & 31.3 & 40.1 & \textbf{54.7} & 31.5$\pm$0.2 \\ \midrule
total & 408.7 & 490.1 & 504.6 & 456.6 & 256.5 & \textbf{546.5} & \textbf{548.7} \\ \midrule
hyperparameter tuning & \textbf{w/o} & w/ & w/ & w/ & w/ & w/ & \textbf{w/o} \\ 
\bottomrule
\end{tabular}%
\end{table}

We also evaluate \algo on the NeoRL benchmark~\citep{qin2022neorl}. NeoRL is a benchmark designed to simulate real-world scenarios by collecting datasets using a more conservative policy, aligning closely with realistic data collection scenarios. The narrow and limited data makes it challenging for offline RL algorithms. The results are shown in \cref{apptab:NeoRL}. The results of baselines are taken directly from the MOBILE paper~\citep{sun2023model}. According to Appendix C in \citep{sun2023model}, these results are obtained by tuning hyperparameters per dataset. For SCAS, we use the same fixed set of hyperparameters as specified in \cref{appsec:exp_details}. Without additional hyperparameter tuning, SCAS still performs comparably to MOBILE and outperforms other baselines in total scores.

\subsection{Comparisons with Additional Baselines}

\begin{table}[h]
  \caption{Comparisons with additional baselines on the D4RL benchmark. Here SCAS-ht means SCAS with slight hyperparameter tuning, selecting $\lambda$ from $\{0.025,0.25\}$. The results of SCAS-ht are averaged over 5 random seeds and the others are taken from their papers.}
  \label{apptab:SCASht}
  \small
  \centering
  \setlength{\tabcolsep}{5.0pt}
  \begin{tabular}{@{}l|cccccc|cc@{}}
    \toprule
    \multirow{2}{*}{Dataset}                               & \multicolumn{6}{c}{\begin{tabular}[c]{@{}c@{}}Ensemble-free\\\end{tabular}} & \multicolumn{2}{c}{\begin{tabular}[c]{@{}c@{}}Ensemble-based\\\end{tabular}} \\ 
\cmidrule(r){2-7}\cmidrule(r){8-9}
& DW+CQL & DW+IQL & SQL & EQL & DQL & SCAS-ht & EDAC & RORL \\ 
    \midrule
    halfcheetah-med & 46.5 & 47.7 & 48.3 & 47.2 & 51.1 & \textbf{58.5$\pm$1.1} & 65.9 & \textbf{66.8} \\ 
    hopper-med & 66.1 & 62.5 & 75.5 & 74.6 & 90.5 & \textbf{102.5$\pm$0.3} & 101.6 & \textbf{104.8} \\ 
    walker2d-med & 82.1 & 80.8 & 84.2 & 83.2 & 87 & \textbf{90.8$\pm$2.6} & 92.5 & \textbf{102.4} \\ 
    halfcheetah-med-rep & 45.1 & 44.6 & 44.8 & 44.5 & 47.8 & \textbf{52.9$\pm$1.4} & 61.3 & \textbf{61.9} \\ 
    hopper-med-rep & 88.6 & 79.7 & 99.7 & 98.1 & 101.3 & \textbf{101.6$\pm$1.0} & 101.0 & \textbf{102.8} \\ 
    walker2d-med-rep & 75.3 & 65.1 & 81.2 & 76.6 & \textbf{95.5} & 88.1$\pm$4.2 & 87.1 & \textbf{90.4} \\ 
    halfcheetah-med-exp & 86.1 & 93.7 & 94.0 & 90.6 & \textbf{96.8} & 91.7$\pm$2.7 & 106.3 & \textbf{107.8} \\ 
    hopper-med-exp & 92.9 & 81.0 & \textbf{111.8} & 105.5 & 111.1 & 109.7$\pm$3.5 & 110.7 & \textbf{112.7} \\ 
    walker2d-med-exp & 109.7 & 109.7 & 110.0 & 110.2 & 110.1 & \textbf{110.8$\pm$1.0} & 114.7 & \textbf{121.2} \\ 
    \midrule
    locomotion total & 692.4 & 664.8 & 749.5 & 730.5 & 791.2 & \textbf{806.6} & 841.1 & \textbf{870.8} \\ 
    \midrule
    antmaze-umaze & 72.7 & 81.3 & 92.2 & 93.2 & \textbf{93.4} & 90.4$\pm$3.6 & 0.0 & \textbf{96.7} \\ 
    antmaze-umaze-div & 34.0 & 61.0 & \textbf{74.0} & 65.4 & 66.2 & 66.4$\pm$14.3 & 0.0 & \textbf{90.7} \\ 
    antmaze-med-play & 4.0 & 78.7 & 80.2 & 77.5 & 76.6 & \textbf{83.6$\pm$3.1} & 0.0 & \textbf{76.3} \\ 
    antmaze-med-div & 1.3 & 64.7 & 79.1 & 70.0 & 78.6 & \textbf{84.6$\pm$5.0} & 0.0 & \textbf{69.3} \\ 
    antmaze-large-play & 2.0 & 40.0 & 53.2 & 45.6 & 46.4 & \textbf{59.4$\pm$5.0} & 0.0 & \textbf{16.3} \\ 
    antmaze-large-div & 0.0 & 42.0 & 52.3 & 42.5 & \textbf{56.6} & 56.2$\pm$5.4 & 0.0 & \textbf{41.0} \\ 
    \midrule
    antmaze total & 114.0 & 367.7 & 431.0 & 394.2 & 417.8 & \textbf{440.6} & 0.0 & \textbf{390.3} \\ 
    \bottomrule
  \end{tabular}
\end{table}

The original SCAS requires only one single hyperparameter configuration in implementations. For a fair comparison with DW~\citep{hong2023beyond}, EDAC~\citep{an2021uncertainty}, RORL~\citep{yang2022rorl}, SQL~\citep{xu2023offline}, and EQL~\citep{xu2023offline}, we roughly select $\lambda$ from \{0.025, 0.25\} for each dataset, referring to this variant as SCAS-ht. The results of SCAS-ht and these methods are reported in \cref{apptab:SCASht}. Among the ensemble-free methods, SCAS-ht achieves the highest performance in both mujoco locomotion and antmaze domains. Compared with ensemble-based methods, SCAS-ht also performs better on antmaze tasks. 
DW~\citep{hong2023beyond} reweights ID data points by their values for behavior regularization and does not account for OOD states during the test phase. In contrast, our approach considers an OOD state correction scenario, resulting in enhanced robustness during the test phase and better performance.

\subsection{Results of Combining \algo Regularizer into Various Offline RL Objectives}
\label{appsec:combine}

\begin{table}[h]
  \caption{
  Comparisons on the D4RL benchmark.
  Here +SCAS means adding the SCAS regularizer. The results are averaged over 5 random seeds.
  }
  \label{apptab:combine}
  \centering
  \setlength{\tabcolsep}{6.0pt}
  \begin{tabular}{@{}l |cc|cc|cc| c@{}}
    \toprule
    \multirow{2}{*}{Dataset}&\multirow{2}{*}{CQL}& CQL &\multirow{2}{*}{TD3BC} & TD3BC &\multirow{2}{*}{IQL} & IQL & \multirow{2}{*}{SCAS} \\
     &  & +SCAS & & +SCAS &  & +SCAS & \\ 
    \midrule
    halfcheetah-med & \textbf{47.0} & 46.5 & \textbf{48.3} & 44.1 & \textbf{47.4} & 46.8 & 46.6 \\ 
    hopper-med & 53.0 & \textbf{96.1} & 59.3 & \textbf{66.6} & 66.2 & \textbf{76.8} & \textbf{102.5} \\ 
    walker2d-med & 73.3 & \textbf{84.9} & \textbf{83.7} & 81.9 & 78.3 & \textbf{84.0} & \textbf{82.3} \\ 
    halfcheetah-med-rep & \textbf{45.5} & 43.6 & \textbf{44.6} & 40.5 & \textbf{44.2} & \textbf{44.2} & \textbf{44.0} \\ 
    hopper-med-rep & 88.7 & \textbf{100.2} & 60.9 & \textbf{79.4} & 94.7 & \textbf{102.3} & \textbf{101.6} \\ 
    walker2d-med-rep & \textbf{81.8} & 78.6 & \textbf{81.8} & 76.2 & 73.8 & \textbf{76.2} & 78.1 \\ 
    halfcheetah-med-exp & 75.6 & \textbf{92.9} & \textbf{90.7} & 89.6 & 86.7 & \textbf{92.7} & 91.7 \\ 
    hopper-med-exp & 105.6 & \textbf{108.2} & 98.0 & \textbf{108.9} & 91.5 & \textbf{101.9} & \textbf{109.7} \\ 
    walker2d-med-exp & \textbf{107.9} & 104.3 & \textbf{110.1} & 106.0 & \textbf{109.6} & 105.4 & 108.4 \\ 
    \midrule
    total & 678.4 & \textbf{755.5} & 677.4 & \textbf{693.2} & 692.4 & \textbf{730.3} & \textbf{764.9} \\ 
    \bottomrule
  \end{tabular}
\end{table}

The SCAS regularizer is compatible with various offline RL objectives. We conduct experiments to combine SCAS with CQL~\citep{kumar2020conservative}, IQL~\citep{kostrikov2022offline}, and TD3BC~\citep{fujimoto2021minimalist}. Comparisons between these combined algorithms and the original ones are shown in \cref{apptab:combine}. We find that applying the SCAS regularizer leads to improved performance for these popular algorithms, which could be attributed to the OOD state correction effects of SCAS. However, we also find that these combined methods do not achieve better performance than the original SCAS (comparable on most tasks and worse on some tasks). We hypothesize that this is because SCAS already has the effect of OOD action suppression, and when combined with offline RL objectives that also aim for OOD action suppression, it may become overly conservative. As a result, the combined algorithms may perform worse than the original SCAS on some sub-optimal datasets.

\subsection{Additional Parameter Study Results}
\label{appsec:ablation}

In this section, we present additional parameter study results conducted on four challenging Antmaze tasks, including antmaze-large-play-v2, antmaze-large-diverse-v2, antmaze-medium-play-v2, and antmaze-medium-diverse-v2.

\textbf{Inverse Temperature $\alpha  $.}
The inverse temperature $\alpha  $ is the key hyperparameter in \algo for achieving value-aware OOD state correction. It controls the significance of the values of next states in \algo's OOD state correction. If $\alpha  =0$, the effect corresponds to vanilla OOD state correction. As $\alpha  $ gets larger, \algo is more inclined to correct the agent to the high-value ID states. Thus we can assess the effectiveness of value-aware OOD state correction compared to vanilla OOD state correction by varying $\alpha  $. Here we test \algo with different $\alpha  $ and the results are shown in \cref{appfig:temp}. We observe that a large $\alpha$ is crucial for achieving good performance on all the antmaze tasks, clearly demonstrating the effectiveness of our \textit{value-aware} OOD state correction. 
However, too large $\alpha  $~($\alpha  =10$) induces less satisfying performance, probably due to the increased variance of the learning objective. In general, we find that choosing $\alpha=5$ leads to the best performance.

\begin{figure}[h]
	\centering
	\includegraphics[width=\textwidth]{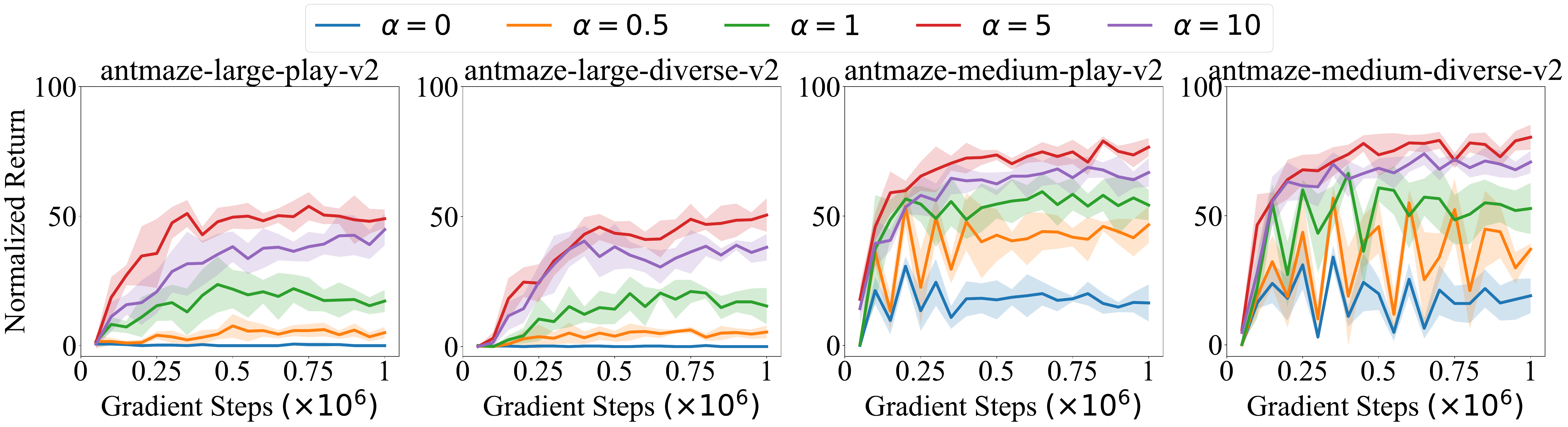}
\caption{
Additional results from the parameter study on the inverse temperature $\alpha$. The curves are averaged over 5 random seeds, with the shaded area representing the standard deviation across seeds.
}
\label{appfig:temp}
\end{figure}

\textbf{Balance Coefficient $\lambda$.}
The balance coefficient $\lambda$ controls the balance between vanilla policy improvement and \algo regularization. If we set $\lambda = 0$, \algo degenerates into the vanilla off-policy RL algorithm. Here we vary $\lambda$ in $\{0,0.25,0.5,0.75,1\}$ and present the corresponding learning curves of \algo in \cref{appfig:lam}. Notably, \algo is able to converge to good performance over a very wide range of $\lambda$. However, if $\lambda=0$, the vanilla off-policy RL suffers from extrapolation error and overestimation, demonstrating poor performance. We also observe a very interesting fact that even when $\lambda=1$ and the signal from RL improvement~(max Q) is removed, \algo still performs well on most tasks. This could be attributed to the fact that value-aware OOD state correction implies some sort of improvement in policy by maximizing the values of policy-induced next states.

\begin{figure}[h]
	\centering
	\includegraphics[width=\textwidth]{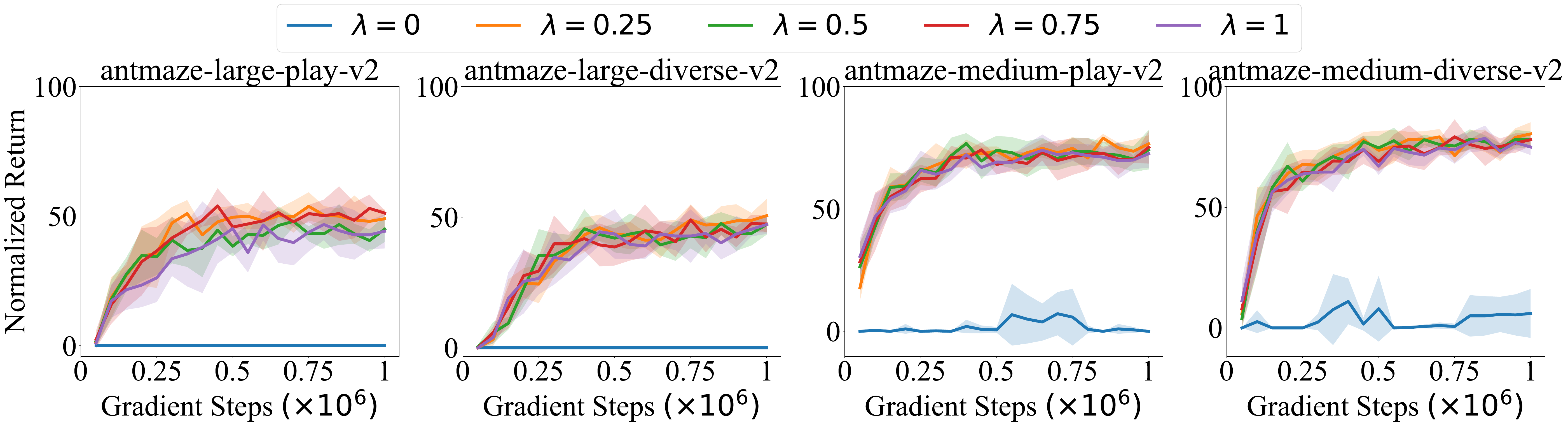}
\caption{
Additional results from the parameter study on the balance coefficient $\lambda$. The curves are averaged over 5 random seeds, with the shaded area representing the standard deviation across seeds.
}
\label{appfig:lam}
\end{figure}

\textbf{Noise Scale $\sigma$.}
The noise scale $\sigma$ is the standard deviation of the Gaussian noise added to the original states for formulating the \algo regularizer. Here we test \algo with different $\sigma$ and present the corresponding learning curves in \cref{appfig:beta}. We observe a significant performance drop with too large $\sigma$~($\sigma=1$) on all the tasks, due to the heavily corrupted learning signal. 
On the other hand, when $\sigma=0$~(without noise), the performance is also less satisfying. With $\sigma=0$, \algo is still able to prevent the agent at ID states from entering OOD states, maintaining the agent in safe regions, but it cannot correct the agent from OOD states to ID states as reliably as the original \algo.
In general, we find that choosing $\sigma = 0.001$ or $0.01$ leads to the best performance.

\begin{figure}[h]
	\centering
	\includegraphics[width=\textwidth]{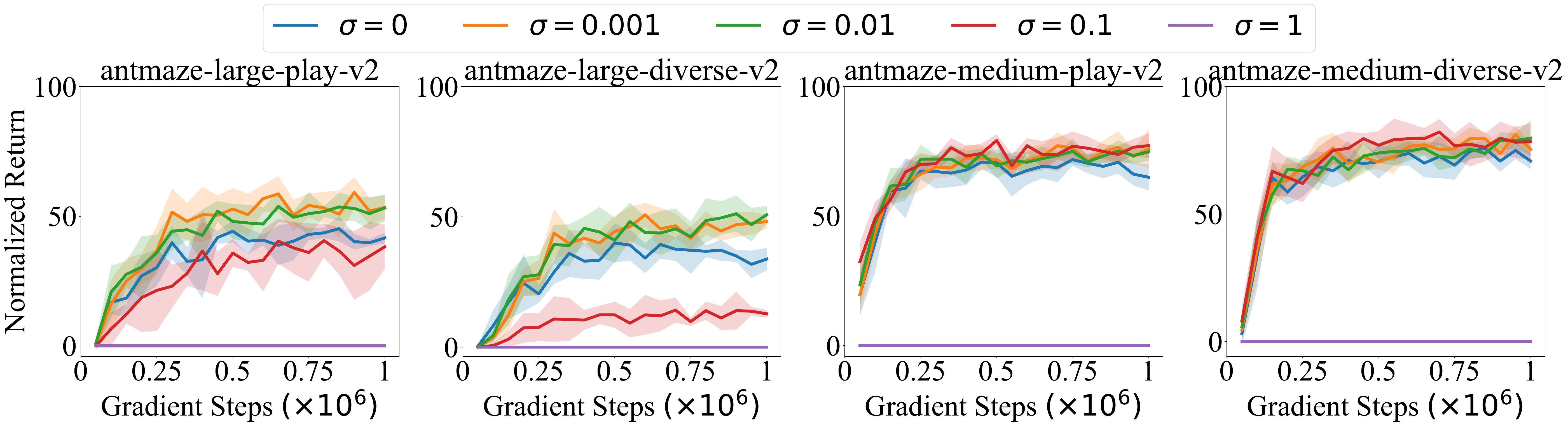}
\caption{
Additional results from the parameter study on the noise scale $\sigma$. 
The curves are averaged over 5 random seeds, with the shaded area representing the standard deviation across seeds.
}
\label{appfig:beta}
\end{figure}

\subsection{Sensitivity Analysis on Dynamics Model Errors}
\label{appsec:checkpoint}

\begin{figure}[h]
    \centering
    \includegraphics[width=0.65\textwidth]{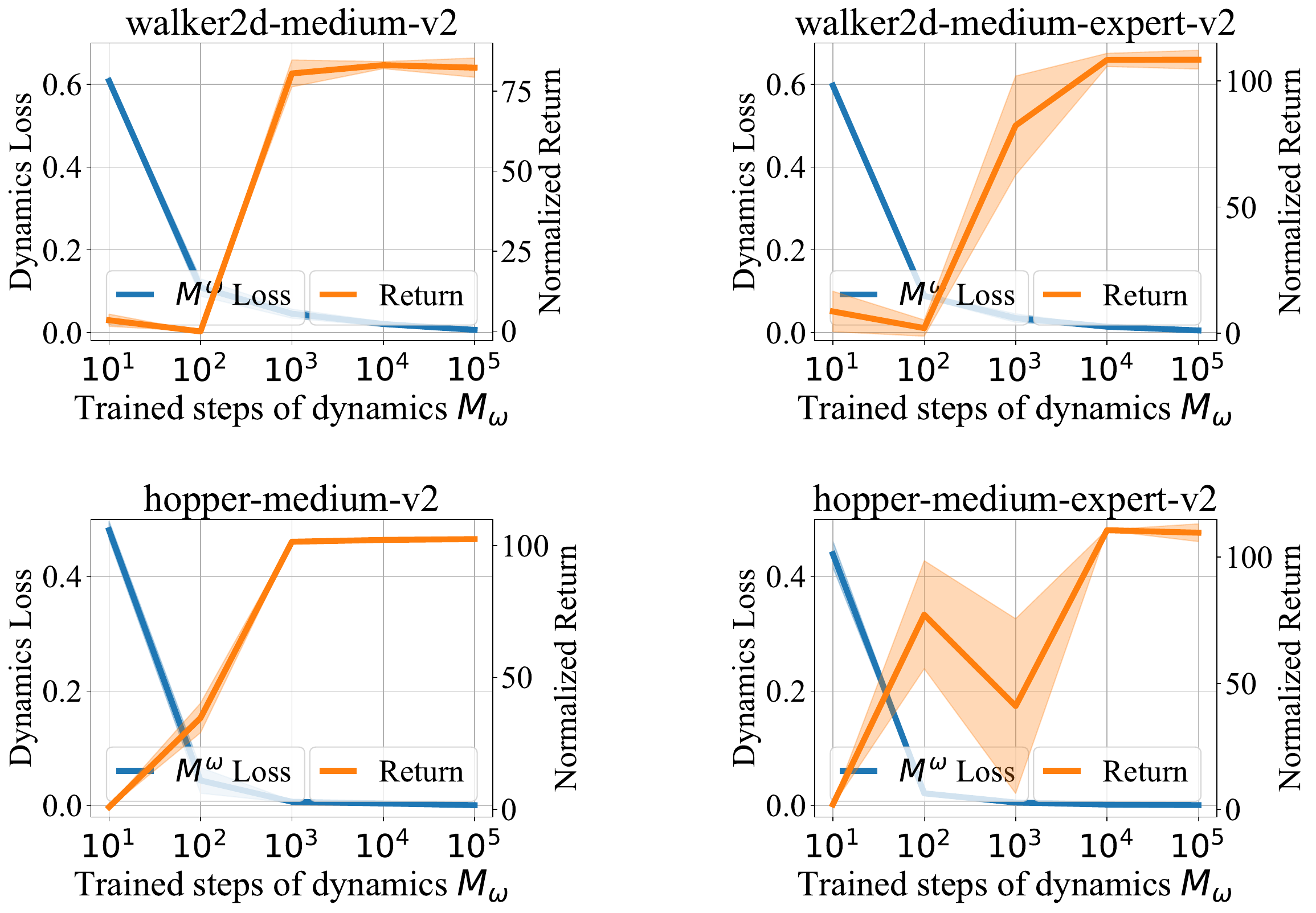}
    \caption{
    Performance of SCAS under different dynamics model checkpoints, which are obtained at different steps in the dynamics model training process. The figure plots the training loss of the dynamics model $M_\omega$ and the corresponding normalized return of SCAS  over 5 random seeds.
    }
    \label{appfig:model error}
\end{figure}

To empirically investigate SCAS under different dynamics model errors, we run SCAS using different checkpoints of the trained dynamics model, which are obtained at different steps in the dynamics model training process. The model error is controlled by the number of trained steps. The results are shown in \cref{appfig:model error}. The figure plots the training loss of the dynamics model $M_\omega$ and the corresponding normalized return of SCAS over 5 random seeds. We observe that the performance of SCAS increases with the number of trained steps of the dynamics model (i.e. the accuracy of the model) and stabilizes at a high level.

\subsection{Learning Curves of \algo}
\label{appsec:learning curves}
Learning curves on Gym locomotion tasks and Antmaze tasks are presented in \cref{fig:mujoco_app} and \cref{fig:ant_app} respectively. The curves are averaged over 5 random seeds, with the shaded area representing the standard deviation across seeds.

\section{Broader Impact}
\label{app_sec:Broader Impact}
Offline RL holds promise for facilitating practical RL applications in domains like robotics, healthcare, and education, where data collection is often costly or risky. However, it is important to recognize its potential negative societal impacts. One concern is that biases in offline data may transfer to the learned policy. In addition, offline RL may affect employment by automating tasks traditionally performed by humans, like factory automation or autonomous driving. Addressing these challenges will contribute to the responsible development and deployment of offline RL algorithms.

From an academic standpoint, this research systematically analyze the OOD state issue in offline RL and propose \algo, a simple yet effective approach that unifies OOD state correction and OOD action suppression. This work potentially offers researchers a new perspective on analyzing the OOD state issue and enhancing test-time robustness in offline RL. Besides, \algo also holds the promise to be extended to safe RL~\citep{achiam2017constrained,gu2022review,garcia2015comprehensive}, meta RL~\citep{finn2017model,wang2022model,wang2024simple,wang2024robust,beck2023survey}, and multi-agent RL~\citep{lowe2017multi,rashid2020monotonic,shao2023complementary,qu2024choices,gronauer2022multi}.

\begin{figure}[ht]
\vspace{10mm}
	\centering
	\includegraphics[width=\linewidth]{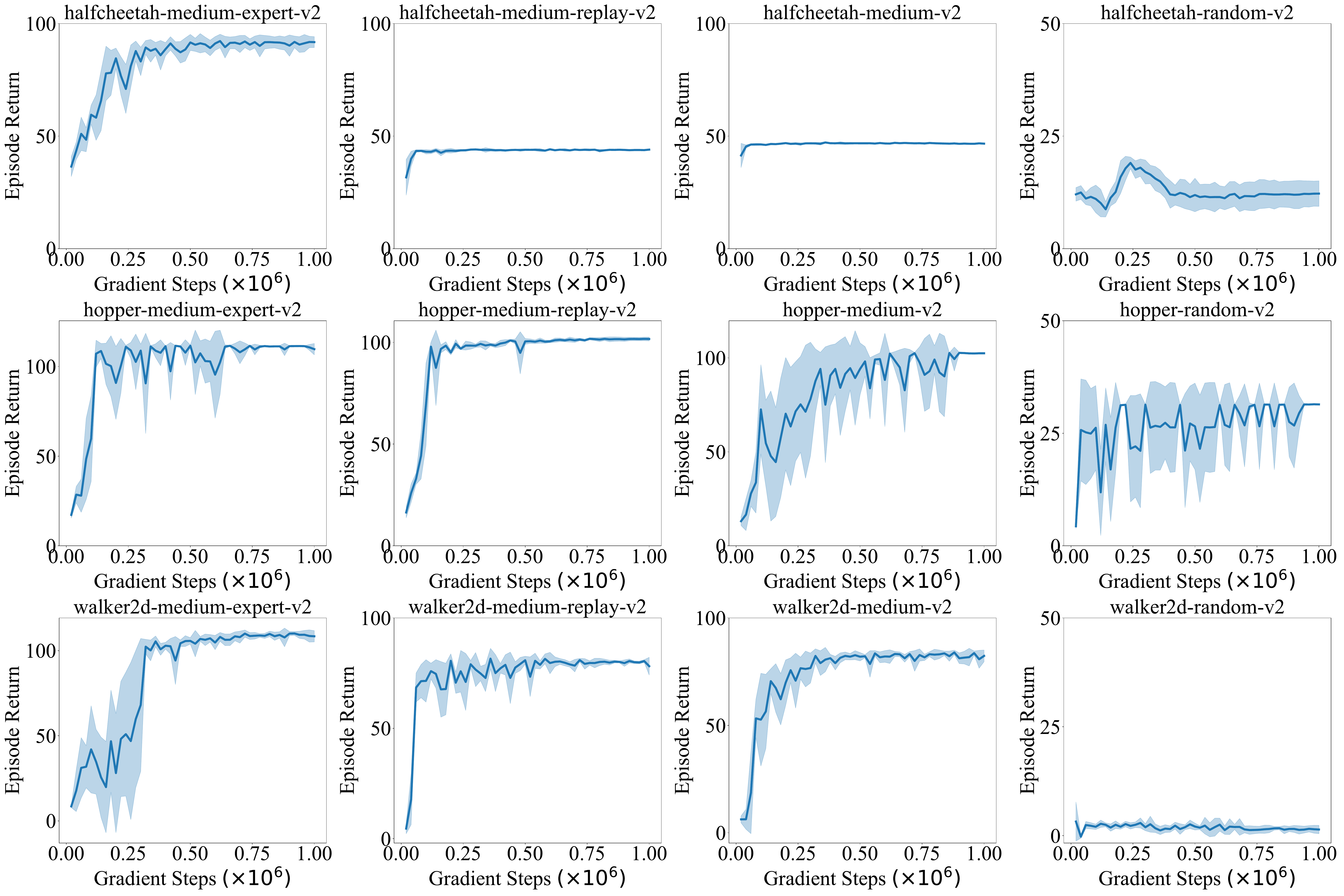}
	\caption{Learning curves of \algo on Gym locomotion tasks. The curves are averaged over 5 random seeds, with the shaded area representing the standard deviation across seeds.
	} 
	\label{fig:mujoco_app}
\end{figure}
\begin{figure}[ht]
	\centering
	\includegraphics[width=0.92\linewidth]{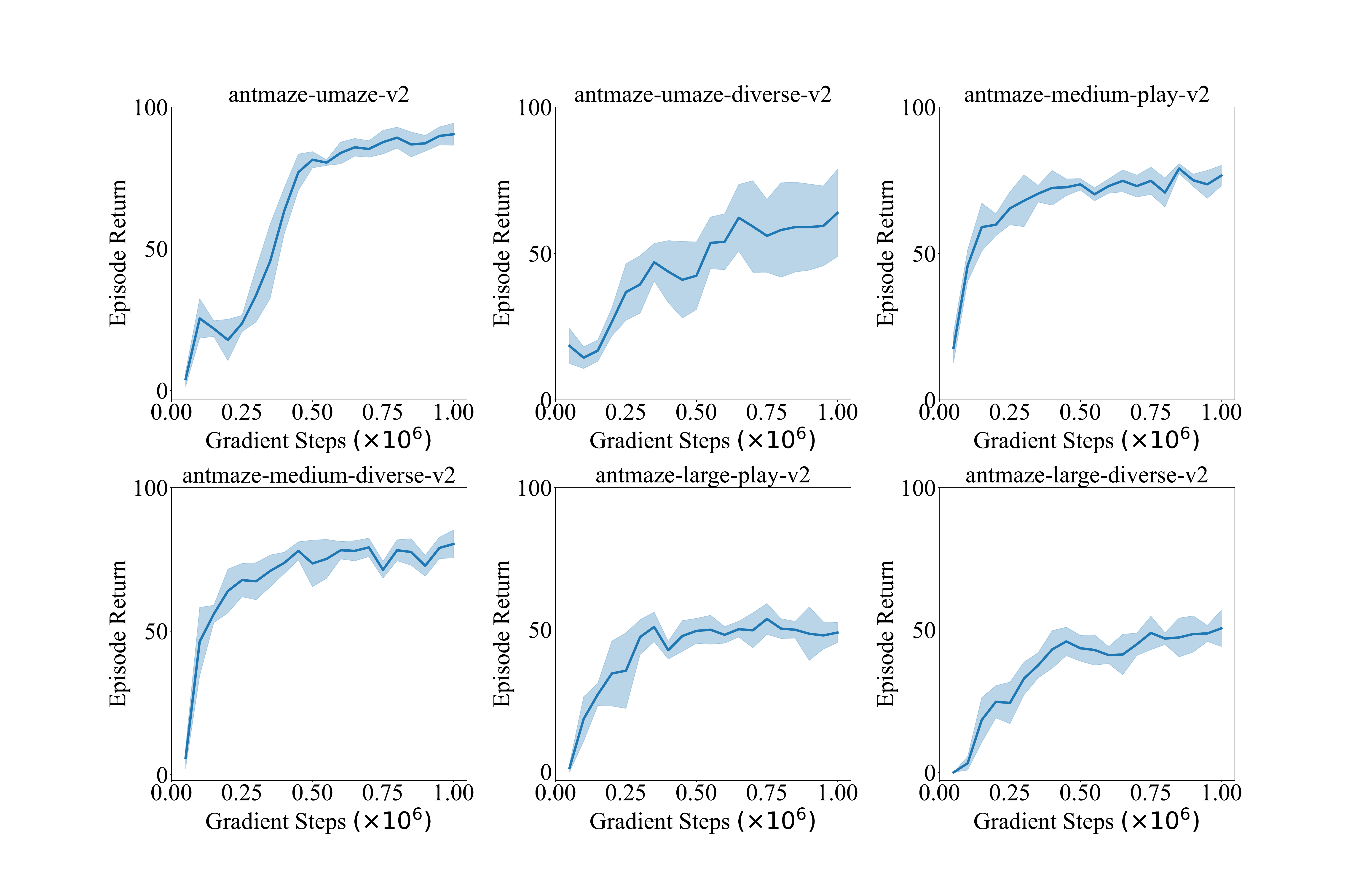}
	\caption{Learning curves of \algo on AntMaze tasks. The curves are averaged over 5 random seeds, with the shaded area representing the standard deviation across seeds.
	}
	\label{fig:ant_app}
\end{figure}

\clearpage


\newpage
\section*{NeurIPS Paper Checklist}

\begin{enumerate}

\item {\bf Claims}
    \item[] Question: Do the main claims made in the abstract and introduction accurately reflect the paper's contributions and scope?
    \item[] Answer: \answerYes{} 
    \item[] Justification: The main claims made in the abstract and introduction accurately reflect the paper's contributions and scope.

\item {\bf Limitations}
    \item[] Question: Does the paper discuss the limitations of the work performed by the authors?
    \item[] Answer: \answerYes{} 
    \item[] Justification: Please refer to \cref{sec:Conclusion and Limitations}.

\item {\bf Theory Assumptions and Proofs}
    \item[] Question: For each theoretical result, does the paper provide the full set of assumptions and a complete (and correct) proof?
    \item[] Answer: \answerYes{} 
    \item[] Justification: Please refer to \cref{appsec:proofs}.

    \item {\bf Experimental Result Reproducibility}
    \item[] Question: Does the paper fully disclose all the information needed to reproduce the main experimental results of the paper to the extent that it affects the main claims and/or conclusions of the paper (regardless of whether the code and data are provided or not)?
    \item[] Answer: \answerYes{} 
    \item[] Justification: Please refer to \cref{appsec:exp_details}.

\item {\bf Open access to data and code}
    \item[] Question: Does the paper provide open access to the data and code, with sufficient instructions to faithfully reproduce the main experimental results, as described in supplemental material?
    \item[] Answer: \answerYes{} 
    \item[] Justification: Please refer to the code in the supplemental material.

\item {\bf Experimental Setting/Details}
    \item[] Question: Does the paper specify all the training and test details (e.g., data splits, hyperparameters, how they were chosen, type of optimizer, etc.) necessary to understand the results?
    \item[] Answer: \answerYes{} 
    \item[] Justification: Please refer to \cref{appsec:exp_details}.

\item {\bf Experiment Statistical Significance}
    \item[] Question: Does the paper report error bars suitably and correctly defined or other appropriate information about the statistical significance of the experiments?
    \item[] Answer: \answerYes{} 
    \item[] Justification: The results in the paper are accompanied by standard deviations across multiple seeds.

\item {\bf Experiments Compute Resources}
    \item[] Question: For each experiment, does the paper provide sufficient information on the computer resources (type of compute workers, memory, time of execution) needed to reproduce the experiments?
    \item[] Answer: \answerYes{} 
    \item[] Justification: Please refer to \cref{appsec:exp_details}.
    
\item {\bf Code Of Ethics}
    \item[] Question: Does the research conducted in the paper conform, in every respect, with the NeurIPS Code of Ethics \url{https://neurips.cc/public/EthicsGuidelines}?
    \item[] Answer: \answerYes{} 
    \item[] Justification: The research conducted in the paper conforms, in every respect, with the NeurIPS Code of Ethics.

\item {\bf Broader Impacts}
    \item[] Question: Does the paper discuss both potential positive societal impacts and negative societal impacts of the work performed?
    \item[] Answer: \answerYes{} 
    \item[] Justification: Please refer to \cref{app_sec:Broader Impact}.
    
\item {\bf Safeguards}
    \item[] Question: Does the paper describe safeguards that have been put in place for responsible release of data or models that have a high risk for misuse (e.g., pretrained language models, image generators, or scraped datasets)?
    \item[] Answer: \answerNA{} 
    \item[] Justification: The paper poses no such risks.

\item {\bf Licenses for existing assets}
    \item[] Question: Are the creators or original owners of assets (e.g., code, data, models), used in the paper, properly credited and are the license and terms of use explicitly mentioned and properly respected?
    \item[] Answer: \answerYes{} 
    \item[] Justification: The creators or original owners of assets used in the paper are properly credited and the license and terms of use are explicitly mentioned and properly respected.

\item {\bf New Assets}
    \item[] Question: Are new assets introduced in the paper well documented and is the documentation provided alongside the assets?
    \item[] Answer: \answerYes{} 
    \item[] Justification: The code is well documented and anonymized.

\item {\bf Crowdsourcing and Research with Human Subjects}
    \item[] Question: For crowdsourcing experiments and research with human subjects, does the paper include the full text of instructions given to participants and screenshots, if applicable, as well as details about compensation (if any)? 
    \item[] Answer: \answerNA{} 
    \item[] Justification: The paper does not involve crowdsourcing nor research with human subjects.

\item {\bf Institutional Review Board (IRB) Approvals or Equivalent for Research with Human Subjects}
    \item[] Question: Does the paper describe potential risks incurred by study participants, whether such risks were disclosed to the subjects, and whether Institutional Review Board (IRB) approvals (or an equivalent approval/review based on the requirements of your country or institution) were obtained?
    \item[] Answer: \answerNA{} 
    \item[] Justification: The paper does not involve crowdsourcing nor research with human subjects.

\end{enumerate}

\end{document}